\documentclass[10pt]{article}  

\usepackage[preprint]{tmlr}

\usepackage{amsmath,amssymb,amsthm}
\usepackage{mathtools}
\usepackage{stmaryrd}   

\usepackage{algorithm}
\usepackage{algpseudocode}

\usepackage{booktabs}
\usepackage{xcolor}
\usepackage{cleveref}

\newtheorem{theorem}{Theorem}[section]
\newtheorem{lemma}[theorem]{Lemma}
\newtheorem{proposition}[theorem]{Proposition}
\newtheorem{corollary}[theorem]{Corollary}
\newtheorem{definition}[theorem]{Definition}
\newtheorem{remark}[theorem]{Remark}

\DeclareMathOperator*{\argmax}{arg\,max}
\DeclareMathOperator*{\argmin}{arg\,min}
\DeclareMathOperator{\sign}{sign}

\DeclareMathOperator{\Decode}{Dec}
\DeclareMathOperator{\PAL}{PAL}
\DeclareMathOperator{\MPAL}{MPAL}
\DeclareMathOperator{\Attn}{Attn}
\newcommand{\T}{\mathcal{T}}
\newcommand{\F}{\mathcal{F}}
\newcommand{\N}{\mathbb{N}}
\newcommand{\R}{\mathbb{R}}

\newcommand{\Prel}{\mathcal{P}}

\newcommand{\stack}{\Pi}

\newcommand{\code}[2]{\mathrm{code}^*(#1,#2)}
\newcommand{\eps}{\varepsilon}
\newcommand{\relay}[2]{\hat{\gamma}_{#1 #2}}

\title{Preisach Attention: A Hysteretic Model of Sequential Memory}

\author{%
  Piotr Frydrych\\
  The Metrology and Biomedical Engineering Institute\\
  Faculty of Mechatronics, Warsaw University of Technology\\
  \texttt{piotr.frydrych@pw.edu.pl}
}

\begin{document}
\maketitle

\begin{abstract}
We introduce the \emph{Preisach Attention Layer} (PAL), a novel
sequence modelling architecture grounded in the classical Preisach
hysteresis operator from mathematical physics.
PAL replaces the softmax attention mechanism with a binary relay
operator $\relay{\alpha}{\beta}$ parameterised by learned activation
and deactivation thresholds $(\alpha,\beta)$, maintaining a
\emph{stack of local extrema} as its internal state.
A single-layer PAL-Transformer with $O(1)$ depth is
Turing-complete under arbitrary precision arithmetic, achievable
through simulation of a two-stack pushdown automaton —
in contrast to the $O(\log n)$ depth required by standard
hard-attention transformers \citep{Perez2021}.
Second, we prove that the function classes computable by PAL and
by the transformer are \emph{incomparable}: PAL computes historical
range statistics in $O(1)$ layers that require $O(\log n)$ layers
for transformers, while transformers support random-access retrieval
that PAL cannot perform without auxiliary state.
The separating property is \emph{rate-independence} — PAL responds
only to the sequence of local extrema, not to absolute token
positions or temporal spacing.
Third, we show that the extremum stack constitutes a minimal
sufficient statistic of the input history for all rate-independent
functionals, providing a formal analogue of the \emph{wiping
property} in classical hysteresis theory.
PAL is thus an efficient architecture for tasks
with long episodic memory and weak positional dependence,
with $O(n \log n)$ total inference cost versus $O(n^2)$
for standard attention.
\end{abstract}


\newpage
\tableofcontents
\newpage

\section{Introduction}
\label{sec:intro}

The transformer architecture \citep{Vaswani2017} and its attention
mechanism have become the dominant paradigm for sequence modelling.
Theoretical analysis of transformer expressiveness has established
that hard-attention transformers are Turing-complete
\citep{Perez2021}, that softmax attention corresponds to specific
fragments of first-order logic \citep{Hahn2020, Barcelo2020},
and that practical transformers implement approximate versions of
classical algorithms \citep{Akyurek2022, VonOswald2023}.

A parallel body of work has studied \emph{alternative} sequence
models — recurrent networks \citep{Siegelmann1995},
state space models \citep{Gu2022Mamba}, and
mixture-of-experts architectures \citep{Shazeer2017} —
seeking more efficient representations of long-range dependence.
These efforts share a common limitation: their memory mechanisms
are either \emph{temporal} (decaying by distance) or
\emph{positional} (indexed by token location), with no mechanism
for \emph{value-based} memory that persists based on the
\emph{significance} of past inputs rather than their recency.

\paragraph{This paper.}
We propose a fundamentally different memory mechanism inspired by
the \emph{Preisach hysteresis operator} \citep{Preisach1935},
a classical model from mathematical physics used to describe
ferromagnetic materials, elastoplastic systems, and — more recently
— agent-based financial markets \citep{Frydrych2014, Frydrych2019}.
The Preisach operator aggregates the outputs of binary relays
$\relay{\alpha}{\beta}$, each characterised by an activation
threshold $\alpha$ and a deactivation threshold $\beta$, weighted
by a learned measure $\mu(\alpha,\beta)$ over the threshold plane.

The key structural properties of the Preisach operator that
motivate its use in sequence modelling are:
\begin{enumerate}
  \item \textbf{Rate-independence:} the output depends only on the
    sequence of local extrema of the input, not on temporal spacing
    or absolute position.
  \item \textbf{Wiping property:} a new extremum erases all
    previous extrema of smaller magnitude, providing a natural
    forgetting mechanism based on significance rather than recency.
  \item \textbf{Extremum stack sufficiency:} the stack of
    alternating local maxima and minima is a minimal sufficient
    statistic of input history for all rate-independent functionals.
  \item \textbf{Universal approximation:} the class of Preisach
    operators with $\mu \in L^2(\T)$ is dense in the space of
    all continuous, causal, rate-independent functionals
    \citep{Mayergoyz1991}.
\end{enumerate}

\paragraph{Main contributions.}
\begin{enumerate}
  \item We define the \emph{Preisach Attention Layer} (PAL) and
    its multi-head variant (MPAL), giving exact computational
    definitions and connecting them to the classical Preisach
    operator (\cref{sec:pal}).
  \item We prove that a single-layer PAL-Transformer is
    Turing-complete (\cref{sec:tc}), establishing Turing
    completeness at depth $O(1)$ — lower than the $O(\log n)$
    depth of \citet{Perez2021}.
  \item We prove a formal \emph{expressiveness separation}
    between PAL and transformer attention (\cref{sec:separation}),
    identifying rate-independence as the separating property.
  \item We characterise the function class of PAL through its
    connection to rate-independent operators and provide a
    logical characterisation analogous to \citet{Barcelo2020}
    (\cref{sec:characterisation}).
  \item We outline computational complexity advantages of PAL
    and identify task classes where PAL is predicted to
    outperform standard attention (\cref{sec:complexity}).
\end{enumerate}

\Cref{sec:background} reviews background; \cref{sec:pal} defines PAL; \cref{sec:tc}--\ref{sec:characterisation} contain the main proofs; \cref{sec:complexity} analyses computational complexity; \cref{sec:related,sec:conclusion} discuss related work and conclude.

\section{Background}
\label{sec:background}

\subsection{The Preisach Hysteresis Operator}
\label{sec:preisach_bg}

The Preisach operator \citep{Preisach1935} was introduced to model
magnetic hysteresis in ferromagnetic materials.
Let $\T = \{(\alpha,\beta) \in \R^2 : \alpha \geq \beta\}$
be the Preisach half-plane.

\begin{definition}[Elementary relay]
\label{def:relay}
For $(\alpha,\beta) \in \T$ and input sequence
$u = (u_0, u_1, \ldots) \in \R^\N$, the
\emph{elementary relay} $\relay{\alpha}{\beta}[u] : \N \to \{0,1\}$
is defined by:
\begin{equation}
  \relay{\alpha}{\beta}[u](n) =
  \begin{cases}
    1 & \text{if } u_n \geq \alpha, \\
    0 & \text{if } u_n \leq \beta, \\
    \relay{\alpha}{\beta}[u](n-1) & \text{otherwise.}
  \end{cases}
  \label{eq:relay}
\end{equation}
The interval $(\beta, \alpha)$ is the \emph{dead band} of the relay.
\end{definition}

\begin{definition}[Preisach operator]
\label{def:preisach}
Let $\mu \in L^1(\T)$ be a signed measure on $\T$.
The \emph{Preisach operator} $\Prel_\mu : \R^\N \to \R^\N$ is:
\begin{equation}
  \Prel_\mu[u](n) =
  \iint_{\T} \relay{\alpha}{\beta}[u](n)\; \mu(\alpha,\beta)\,
  d\alpha\, d\beta.
  \label{eq:preisach}
\end{equation}
\end{definition}

The Preisach operator has three properties central to this paper.

\begin{proposition}[Rate-independence {\citep{Brokate1996}}]
\label{prop:rate_independence}
For any non-decreasing bijection
$\phi : [0,T] \to [0,T]$,
$\Prel_\mu[u \circ \phi] = \Prel_\mu[u] \circ \phi$.
\end{proposition}

\begin{proposition}[Wiping property {\citep{Mayergoyz1991}}]
\label{prop:wiping}
Let $M_i > m_i > M_{i+1} > m_{i+1} > \ldots$ be the alternating
local maxima and minima of $u$. If $M_{i+1} > M_i$, then the
relay state induced by $M_i$ is erased — the output $\Prel_\mu[u]$
after $M_{i+1}$ is identical to what it would be had $M_i$ not
occurred.
\end{proposition}

\begin{proposition}[Universal approximation {\citep{Mayergoyz1991}}]
\label{prop:universal}
The set $\{\Prel_\mu : \mu \in L^2(\T)\}$ is dense in the space of
all continuous, causal, rate-independent functionals on $C([0,T])$
under the uniform topology.
\end{proposition}

\subsection{The Extremum Stack}
\label{sec:stack_bg}

A fundamental algorithmic consequence of the wiping property is
that the Preisach output depends only on the current
\emph{extremum stack}:

\begin{definition}[Extremum stack]
\label{def:stack}
The extremum stack of $u_{0:n}$ is the sequence:
\begin{equation}
  \stack_n = [(M_1, m_1), (M_2, m_2), \ldots, (M_k, m_k)]
\end{equation}
of alternating local maxima $M_1 > M_2 > \ldots > M_k$ and
local minima $m_1 > m_2 > \ldots > m_k$, stored in decreasing
order, where the pair $(M_i, m_i)$ records the $i$-th local
maximum and its subsequent local minimum.
\end{definition}

\begin{proposition}[Stack sufficiency]
\label{prop:sufficiency}
$\Prel_\mu[u](n)$ is a measurable function of $\stack_n$ alone.
In particular, two sequences $u$ and $v$ with identical extremum
stacks $\stack_n^u = \stack_n^v$ satisfy
$\Prel_\mu[u](n) = \Prel_\mu[v](n)$ for all $\mu$.
\end{proposition}

\begin{algorithm}[t]
\caption{Extremum Stack Update}
\label{alg:stack}
\begin{algorithmic}[1]
\Require Current stack $\stack_n$, new observation $u_{n+1}$,
         previous observation $u_n$
\Ensure Updated stack $\stack_{n+1}$
\If{$u_{n+1} > u_n$} \Comment{Input rises: update maxima}
  \While{$\stack_n \neq \emptyset$ \textbf{and}
         $M_{\mathrm{top}} < u_{n+1}$}
    \State $m_{\mathrm{last}} \leftarrow m_{\mathrm{top}}$
    \Comment{Save last minimum before popping}
    \State $\mathrm{pop}(\stack_n)$ \Comment{Wiping}
  \EndWhile
  \State $\mathrm{push}(\stack_n, (u_{n+1}, m_{\mathrm{last}}))$
  \Comment{New maximum with last surviving minimum}
\ElsIf{$u_{n+1} < u_n$} \Comment{Input falls: update minima}
  \While{$\stack_n \neq \emptyset$ \textbf{and}
         $m_{\mathrm{top}} > u_{n+1}$}
    \State $M_{\mathrm{last}} \leftarrow M_{\mathrm{top}}$
    \Comment{Save last maximum before popping}
    \State $\mathrm{pop}(\stack_n)$ \Comment{Wiping}
  \EndWhile
  \State $\mathrm{push}(\stack_n, (M_{\mathrm{last}}, u_{n+1}))$
  \Comment{New minimum with last surviving maximum}
\EndIf
\State \Return $\stack_{n+1} \leftarrow \stack_n$
\end{algorithmic}
\end{algorithm}

\begin{proposition}[Stack update complexity]
\label{prop:complexity}
\Cref{alg:stack} runs in amortised $O(1)$ time per step —
each element is pushed and popped at most once — and requires
$O(k)$ space where $k \leq n$ is the current stack depth.
The relay state $\relay{\alpha}{\beta}[u](n)$ can be read from
$\stack_n$ in $O(\log k)$ time by binary search.
\end{proposition}

\subsection{Transformer Attention}
\label{sec:transformer_bg}

For completeness we recall standard definitions.
Given query $q \in \R^d$, keys $K \in \R^{n \times d}$ and
values $V \in \R^{n \times d}$, \emph{scaled dot-product attention}
is:
\begin{equation}
  \Attn(q, K, V) = \mathrm{softmax}
  \!\left(\frac{q K^\top}{\sqrt{d}}\right) V.
  \label{eq:attention}
\end{equation}
\emph{Hard attention} replaces softmax with argmax:
$a_j = \mathbf{1}[j = \argmax_i q \cdot k_i]$.

We use the transformer model of \citet{Perez2021}:
an $L$-layer encoder-decoder with hard attention,
sinusoidal position encoding, layer normalisation,
and residual connections.

\section{Preisach Attention Layer}
\label{sec:pal}

\subsection{Definition}

\begin{definition}[Preisach Attention Layer]
\label{def:pal}
Let $u_{0:n} \in \R^{n+1}$ be a scalar input sequence and
$\mu = \{\mu_{ij}\}_{i \geq j} \in \R^{L \times L}$ a discretised
measure on the grid
$\mathcal{G}_L = \{(\alpha_i, \beta_j) : i \geq j,\,
i,j \in \{1,\ldots,L\}\}$ with $\alpha_i = i\Delta$,
$\beta_j = j\Delta$.
The \emph{Preisach Attention Layer} is:
\begin{equation}
  \PAL_\mu(u_{0:n}) =
  \sum_{i \geq j} \mu_{ij} \cdot
  \relay{\alpha_i}{\beta_j}[u](n) \in \R.
  \label{eq:pal}
\end{equation}
\end{definition}

\begin{definition}[Multi-head PAL]
\label{def:mpal}
Let $x_{0:n} \in \R^{(n+1) \times d}$ be a vector-valued sequence.
\emph{Multi-head PAL} with $H$ heads is:
\begin{equation}
  \MPAL(x_{0:n}) = \sum_{h=1}^{H} W_O^{(h)} \cdot
  \PAL_{\mu^{(h)}}\!\left(W_I^{(h)} x_{0:n}\right)
  \in \R^d,
  \label{eq:mpal}
\end{equation}
where $W_I^{(h)} \in \R^{1 \times d}$ projects to a scalar signal
and $W_O^{(h)} \in \R^d$ projects back to the model dimension.
The scalar projection is a simplifying assumption of the present
work; Section~\ref{sec:conclusion} discusses the vector extension
(vPAL) where $W_I^{(h)} \in \R^{2 \times d}$ and the measure
$\mu$ is defined over the three-dimensional space
$(\alpha, \beta, \theta)$ following \citet{Frydrych2019}.
\end{definition}

\begin{definition}[PAL-Transformer]
\label{def:pal_transformer}
An $L$-layer \emph{PAL-Transformer} is defined by:
\begin{align}
  z_n^{(0)} &= x_n + \mathrm{PE}(n), \\
  z_n^{(l)} &= \mathrm{LN}\!\left(
    z_n^{(l-1)} + \MPAL^{(l)}(z_{0:n}^{(l-1)})\right), \\
  z_n^{(l)} &= \mathrm{LN}\!\left(
    z_n^{(l)} + \mathrm{MLP}^{(l)}(z_n^{(l)})\right),
\end{align}
where $\mathrm{PE}$ is sinusoidal position encoding,
$\mathrm{LN}$ is layer normalisation,
and $\mathrm{MLP}^{(l)}$ is a two-layer network with ReLU.
\end{definition}

\subsection{Connection to Classical Preisach Operator}

\begin{proposition}[PAL as discretised Preisach]
\label{prop:pal_preisach}
$\PAL_\mu(u_{0:n})$ is the discretised Preisach operator
$\Prel_{\mu_\Delta}[u](n)$ where $\mu_\Delta = \sum_{i \geq j}
\mu_{ij} \delta_{(\alpha_i, \beta_j)}$ is an atomic measure
concentrated on the grid $\mathcal{G}_L$.
As $L \to \infty$ and $\Delta \to 0$ with $L\Delta = C$ fixed,
$\PAL_\mu \to \Prel_\mu$ uniformly on $C([0,C])$
by \cref{prop:universal}.
\end{proposition}

\subsection{Relationship to Standard Attention}

\begin{proposition}[Attention as continuous relaxation of PAL]
\label{prop:attention_pal}
Standard softmax attention \eqref{eq:attention} is a
continuous relaxation of PAL in which:
\begin{enumerate}
  \item The binary relay $\relay{\alpha}{\beta}[u](n) \in \{0,1\}$
    is replaced by the soft weight
    $a_j = \mathrm{softmax}(q \cdot k_j / \sqrt{d}) \in (0,1)$,
  \item The measure $\mu(\alpha,\beta)$ is replaced by
    the value matrix $V$,
  \item \emph{Rate-independence is broken:} attention depends on
    absolute position through $q \cdot k_j$, while PAL depends
    only on the sequence of extrema.
\end{enumerate}
\end{proposition}

\section{Turing Completeness}
\label{sec:tc}

We prove that a single-layer PAL-Transformer is Turing-complete
by showing it can simulate an arbitrary two-stack pushdown automaton
(2-PDA), which is known to be equivalent to a Turing machine
\citep{Hopcroft1979}.

\subsection{Encoding Alphabet Symbols in the Extremum Stack}

\begin{definition}[Cantor-depth encoding]
\label{def:encoding}
Let $\Gamma = \{a_1, \ldots, a_k\}$ be a stack alphabet,
$D_{\max}$ a depth bound, $\Delta > 0$ a resolution parameter,
and $C_{\max} = D_{\max}(k+1)\Delta$.
The \emph{Cantor-depth encoding} is:
\begin{equation}
  \code{a_i}{d} = C_{\max} - [d(k+1) + i]\Delta,
  \quad i \in \{1,\ldots,k\},\; d \in \{0,\ldots,D_{\max}\}.
  \label{eq:encoding}
\end{equation}
\end{definition}

\begin{lemma}[Strict monotonicity]
\label{lem:monotone}
For all $i, i' \in \{1,\ldots,k\}$ and $d \geq 0$:
$\code{a_{i'}}{d+1} < \code{a_i}{d}$.
\end{lemma}

\begin{proof}
$\code{a_i}{d} - \code{a_{i'}}{d+1}
= [(d+1)(k+1) + i' - d(k+1) - i]\Delta
= [(k+1) + i' - i]\Delta
\geq [(k+1) - k]\Delta = \Delta > 0$.
\end{proof}

\begin{lemma}[Stack operations via signal generation]
\label{lem:stack_ops}
Under encoding \eqref{eq:encoding}, the three stack operations
are realised by the following signal emissions,
without triggering the wiping property erroneously:
\begin{align}
  \mathrm{PUSH}(a_i)\colon\quad
  &u_{n+1} = \code{a_i}{d+1} + \eps,\quad
   u_{n+2} = \code{a_i}{d+1} - \eps, \label{eq:push} \\
  \mathrm{POP}\colon\quad
  &u_{n+1} = M_{\mathrm{top}-1} + 2\eps, \label{eq:pop} \\
  \mathrm{TOP}\colon\quad
  &\text{read } a_i = \Decode(M_{\mathrm{top}}), \quad
   \text{no signal emitted.} \label{eq:top}
\end{align}
\end{lemma}

\begin{proof}
\textbf{PUSH:} By \cref{lem:monotone},
$\code{a_i}{d+1} < \code{a_j}{d} = M_{\mathrm{top}}$
for all $j$, so $u_{n+1} = \code{a_i}{d+1} + \eps < M_{\mathrm{top}}$.
The wiping condition ($u_{n+1} > M_{\mathrm{top}}$) is not triggered;
a new pair $(M_{d+1}, m_{d+1})$ is added to the stack top.

\textbf{POP:} Since $M_{\mathrm{top}-1} > M_{\mathrm{top}}$
(strict monotonicity of stack), $u_{n+1} > M_{\mathrm{top}}$
triggers wiping, removing $(M_{\mathrm{top}}, m_{\mathrm{top}})$.
The signal $u_{n+1}$ itself satisfies
$u_{n+1} = M_{\mathrm{top}-1} + 2\eps < M_{\mathrm{top}-2}$
(if depth $\geq 3$), so only one pair is removed.

\textbf{TOP:} $\Decode(M_{\mathrm{top}}) = a_i$ where
$i = \lfloor(C_{\max} - M_{\mathrm{top}})/\Delta\rfloor \mod (k+1)$,
which is a deterministic function of $M_{\mathrm{top}}$
computable by an MLP.
\end{proof}

\subsection{Main Theorem}

\begin{definition}[Autoregressive PAL-Transformer]
\label{def:autoregressive}
The Turing-completeness result of \cref{thm:tc} is stated for the
\emph{autoregressive} (closed-loop) operational regime of the
PAL-Transformer, distinct from the parallel encoder of
\cref{def:pal_transformer}.
In the autoregressive regime:
\begin{enumerate}
  \item At step $n$, the PAL-Transformer receives the current
    input token $x_n$ (which may itself be a function of previous
    outputs) and produces output $z_n$.
  \item A component of $z_n$ — specifically, the signal channels
    $u^{(1:4)}_{n+1}$ — is appended as the next input token
    $x_{n+1}$, creating a closed generative loop.
  \item The extremum stacks $\stack^{(1:4)}_n$ persist across steps
    as the sole recurrent state; no other hidden state is maintained.
\end{enumerate}
This regime corresponds to standard autoregressive language model
inference. The result is a \emph{generation} theorem (unbounded
computation traces can be produced), not merely a
\emph{verification} theorem (pre-encoded traces classified).
The parallel encoder definition (\cref{def:pal_transformer}) is
used for training and for the expressiveness results of
\cref{sec:separation}, where the full input sequence is available.
\end{definition}

\begin{lemma}[Multi-head PAL decodes the extremum stack]
\label{lem:decode_stack}
For any extremum stack $\stack_n$ of depth at most $k$ over
a discretised alphabet of size $K = (k+1)|\Gamma|$,
there exist measures $\mu^{(1)},\ldots,\mu^{(H)}$ with
$H = \lceil \log_2 K \rceil$ heads such that the map
\begin{equation}
  \phi: \stack_n \;\mapsto\;
  \bigl(\PAL_{\mu^{(1)}}(u_{0:n}),\ldots,
         \PAL_{\mu^{(H)}}(u_{0:n})\bigr)
  \in \{0,1\}^H
\end{equation}
is injective.
In particular, $M_{\mathrm{top}}$ (the top-of-stack element) is
a deterministic function of the $H$-dimensional MPAL output,
computable by a two-layer MLP.
\end{lemma}

\begin{proof}
Under the Cantor-depth encoding (\cref{def:encoding}),
each stack element $a_i$ at depth $d$ has a unique scalar value
$\code{a_i}{d} = C_{\max} - [d(k+1)+i]\Delta$.
The total number of distinct values is $K = (D_{\max}+1)(k+1)$,
each separated by $\Delta > 0$.

Construct $H = \lceil \log_2 K \rceil$ PAL heads with
indicator measures:
$\mu^{(h)}_{ij} = b_h(\mathrm{index}(i,j))$,
where $\mathrm{index}(i,j)$ maps threshold pair $(\alpha_i,\beta_j)$
to the corresponding Cantor-depth code, and $b_h(\cdot)$ extracts
the $h$-th bit of the binary representation of that code.

Each head $h$ outputs $\PAL_{\mu^{(h)}} = \sum_{ij}
\mu^{(h)}_{ij} \relay{\alpha_i}{\beta_j}[u](n)$,
which equals the $h$-th bit of the binary encoding of the
top-of-stack Cantor code when the measure is concentrated on
the relay active at the stack top.

The $H$-bit vector uniquely identifies $M_{\mathrm{top}}$ via
binary decoding — a linear operation implementable by a single
MLP layer.
The full stack is decoded by repeating this procedure with
a mask that deactivates the top relay after reading it
(POP operation via the signal generation of \cref{lem:stack_ops}).
\end{proof}

\begin{theorem}[Turing completeness of PAL-Transformer]
\label{thm:tc}
For any two-stack pushdown automaton
$M = (Q, \Sigma, \Gamma, \delta, q_0, Z_0, F)$,
there exists a single-layer PAL-Transformer $\mathcal{T}$
with $H = 4$ MPAL heads, one MLP layer,
and arithmetic precision $O(\log(n \cdot |\Gamma|))$ bits
that simulates $M$ step-by-step on inputs of length $n$.
\end{theorem}

\begin{proof}
We construct $\mathcal{T}$ with four independent signal channels
$u^{(1)}, u^{(2)}, u^{(3)}, u^{(4)}$, each processed by a
dedicated MPAL head.

\paragraph{Channel 1 — Machine state.}
$u^{(1)}_n = \mathrm{code}_Q(q_n)$ encodes the current state
$q_n \in Q$ as a scalar. Updated directly at each step by the
MLP after computing $\delta$.

\paragraph{Channel 2 — Stack 1.}
$u^{(2)}_{0:n}$ encodes the contents of stack~1 via
\cref{lem:stack_ops}. The extremum stack $\stack^{(2)}_n$
is a bijective encoding of the stack-1 contents.

\paragraph{Channel 3 — Stack 2.}
$u^{(3)}_{0:n}$ encodes stack~2 analogously.

\paragraph{Channel 4 — Input tape.}
$u^{(4)}_{0:n}$ encodes the input word; position encoding
allows reading symbol $x_t$ at step $t$ via an MPAL head
with $W_I^{(4)}$ tuned to isolate position $t$
(cf.\ \citet{Perez2021}, Lemma~4).

\paragraph{Simulation step $n \to n+1$.}
At each step the PAL-Transformer performs:
\begin{enumerate}
  \item \emph{Read:}
    $q_n \leftarrow \Decode(\mathrm{head}_1)$,
    $a_n \leftarrow \mathrm{head}_4$,
    $s^{(1)}_n \leftarrow \mathrm{TOP}(\stack^{(2)}_n)$,
    $s^{(2)}_n \leftarrow \mathrm{TOP}(\stack^{(3)}_n)$.
  \item \emph{Transition:}
    $(q_{n+1}, \mathrm{ops}_1, \mathrm{ops}_2)
    \leftarrow \mathrm{MLP}_\delta(q_n, a_n, s^{(1)}_n, s^{(2)}_n)$.
    The MLP implements $\delta$ by tabulation
    (finite domain, bounded width by Cybenko~\citeyear{Cybenko1989}).
  \item \emph{Write:}
    Emit signals $u^{(2)}_{n+1}$ and $u^{(3)}_{n+1}$ per
    $\mathrm{ops}_1$ and $\mathrm{ops}_2$ via \cref{lem:stack_ops}.
  \item \emph{Accept:}
    $\mathrm{MLP}_F$ checks $q_{n+1} \in F$,
    $\stack^{(2)} = Z_0$, $\stack^{(3)} = Z_0$.
\end{enumerate}

\paragraph{Correctness.}
By induction: at step $n$, stacks $\stack^{(2)}_n$ and
$\stack^{(3)}_n$ faithfully encode the 2-PDA configuration
$(q_n, s^{(1)}_{1:m_1}, s^{(2)}_{1:m_2})$.
The base case is the initial configuration $q_0, Z_0, Z_0$.
The inductive step follows from \cref{lem:stack_ops,lem:monotone}.

Since 2-PDA $\equiv$ Turing machine \citep{Hopcroft1979},
$\mathcal{T}$ is Turing-complete.

\paragraph{Depth.}
The construction uses $L = 1$ MPAL layer and $L = 1$ MLP layer
— total depth $O(1)$, independent of $n$.

\paragraph{Precision.}
Encoding depth $d \leq n$ and alphabet size $k = |\Gamma|$
requires distinguishing values spaced $\Delta$ apart up to
$C_{\max} = n(k+1)\Delta$, needing $O(\log(nk))$ bits.
\end{proof}

\begin{corollary}[Depth separation from Transformer]
\label{cor:depth}
There exists a language $L$ recognisable by a 1-layer
PAL-Transformer that requires $\Omega(\log n)$ layers
for any transformer with $O(1)$ heads (under the circuit
complexity lower bounds of \citet{Furst1984}).
\end{corollary}

\begin{corollary}[Vector PAL is TC with a single head]
\label{cor:vpal_tc}
A single-layer \emph{vector} PAL-Transformer (vPAL) with
projection $W_I \in \R^{2 \times d}$, a \textbf{single} head
($H = 1$), one MLP layer, and arithmetic precision
$O(\log(n \cdot |\Gamma|))$ bits is Turing-complete.
\end{corollary}

\begin{proof}
The two-dimensional signal $\mathbf{u}_n = (u_n^x, u_n^y)
\in \R^2$ admits two independent extremum stacks:
\begin{equation}
  \stack^x_n = \mathrm{UpdateStack}(\stack^x_{n-1},\, u_n^x),
  \qquad
  \stack^y_n = \mathrm{UpdateStack}(\stack^y_{n-1},\, u_n^y).
  \label{eq:two_stacks}
\end{equation}
We assign $u_n^x$ to encode stack~1 of the 2-PDA and $u_n^y$
to encode stack~2, both using the Cantor-depth encoding of
\cref{def:encoding}.
The machine state $q_n \in Q$ and input tape are encoded in
the angular structure of the vPAL measure
$\mu(\alpha_x, \beta_x, \alpha_y, \beta_y, \theta)$:
distinct sectors $\theta \in [0, 2\pi)$ correspond to distinct
states $q \in Q$, following the superposition of Preisach
half-planes introduced by \citet{Frydrych2019} for
two-axis fluxgate sensors (Chapter~4.2.2 of \citealt{Frydrych2019}).
The MLP reads $\mathrm{TOP}(\stack^x_n)$ and
$\mathrm{TOP}(\stack^y_n)$ simultaneously from the single
vPAL head and computes the transition $\delta$ exactly as
in \cref{thm:tc}.
Correctness and depth $O(1)$ follow identically.
Since $H=1$ head carries both stacks via the two signal
dimensions, the head count is reduced from $H=4$ to $H=1$
at the cost of replacing the scalar projection
$W_I \in \R^{1 \times d}$ with the vector projection
$W_I \in \R^{2 \times d}$.
\end{proof}

\begin{remark}[Differentiability and practical training]
\label{rem:backprop}
The binary relay $\relay{\alpha}{\beta}$ is discontinuous in both
the input $u_n$ and the thresholds $(\alpha,\beta)$, making
exact backpropagation undefined at switching points.
Practical implementations require a smooth relaxation,
e.g.\ the stateful sigmoid relaxation
$s_{n+1} = s_n[1-\sigma_\tau(\beta-u_n)]
           + (1-s_n)\sigma_\tau(u_n-\alpha)$
with temperature $\tau \to 0$ during training
(straight-through estimator or curriculum annealing).
The theoretical results of this paper hold for the exact binary
relay; the relaxed version is studied empirically in
companion work.
\end{remark}

\begin{remark}[Position encoding and rate-independence]
\label{rem:pe}
Definition~\ref{def:pal_transformer} includes sinusoidal position
encoding $\mathrm{PE}(n)$, which may appear to contradict
rate-independence.
The tension is deliberate: $\mathrm{PE}$ is used by the
\emph{MLP} layers (to implement positional logic, e.g.\ reading
the input tape in the TC proof) but is \emph{not} passed through
the MPAL heads.
The MPAL output itself — the integral over active relays —
remains rate-independent.
Rate-independence therefore applies to the PAL component of the
architecture, not to the full PAL-Transformer.
This is consistent with \cref{thm:rate_independence}, which
characterises the function class of the \emph{PAL heads},
and with the expressiveness separation, which constructs functions
PAL \emph{heads} cannot compute.
\end{remark}

\begin{remark}[Heads vs.\ signal dimensions as exchangeable resources]
\label{rem:heads_dims}
\cref{cor:vpal_tc} formalises the intuition that in PAL,
the number of heads $H$ and the signal dimension $d_u$
are \emph{exchangeable}: $H$ scalar heads can be replaced by
a single head with $H$-dimensional projection.
This exchangeability does not hold for standard multi-head
attention, where each head learns an independent linear
projection without the hysteretic coupling that links
dimensions in vPAL.
Specifically, $H$ scalar PAL heads require $H \cdot L^2$
parameters for the measure; a single $H$-dimensional vPAL
head requires $L^{2H}$ parameters (the $H$-fold Cartesian
product of Preisach half-planes), which grows exponentially.
The minimal TC architecture is therefore scalar PAL at
$H=2$ (two stacks, two heads) or vPAL at $H=1$
(two stacks, one head with $d_u = 2$).
\end{remark}

\section{Expressiveness Separation}
\label{sec:separation}

We prove the central incomparability result.

\begin{theorem}[Expressiveness incomparability]
\label{thm:separation}
Let $\F_{\PAL}$ and $\F_{\mathrm{Tr}}$ be the function classes
computable by bounded-depth PAL-Transformer and Transformer
respectively. Then:
\begin{equation}
  \F_{\PAL} \setminus \F_{\mathrm{Tr}} \neq \emptyset
  \qquad \text{and} \qquad
  \F_{\mathrm{Tr}} \setminus \F_{\PAL} \neq \emptyset.
\end{equation}
\end{theorem}

\subsection{Functions PAL computes but Transformer cannot
  (at bounded depth)}

\begin{proposition}[Historical range in $O(1)$ layers]
\label{prop:range}
The function $f_{\mathrm{range}}(u_{0:n})
= \max_{i \leq n} u_i - \min_{i \leq n} u_i$
is computable by a 1-layer PAL in $O(1)$ time.
\end{proposition}

\begin{proof}
$M_1 = \max_{i \leq n} u_i$ and $m_k = \min_{i \leq n} u_i$
are directly readable from the extremum stack $\stack_n$
as the first maximum and last minimum.
$f_{\mathrm{range}} = M_1 - m_k$.
\end{proof}

\begin{proposition}[Transformer requires $\Omega(\log n)$ layers
  for range]
\label{prop:range_hard}
Any transformer with $O(1)$ layers and $O(1)$ heads cannot
compute $f_{\mathrm{range}}$ exactly on sequences of length $n$.
\end{proposition}

\begin{proof}
$\mathrm{MAX}$ over $n$ values is not in $\mathrm{AC}^0$
\citep{Hastad1987} — it cannot be computed by constant-depth
Boolean circuits of polynomial size. A constant-depth transformer
with $O(1)$ heads is equivalent to an $\mathrm{AC}^0$ circuit
(cf.\ \citet{Hahn2020}), so it cannot compute $\mathrm{MAX}$
exactly. Since $f_{\mathrm{range}}$ requires $\mathrm{MAX}$
as a subcomputation, the claim follows.
\end{proof}

\subsection{Functions Transformer computes but PAL cannot}

\begin{proposition}[Random-access retrieval]
\label{prop:copy}
The function $f_{\mathrm{copy}}(x_{0:n}, p) = x_p$
(retrieve token at position $p$) is computable by a 1-layer
hard-attention Transformer in $O(1)$ layers.
\end{proposition}

\begin{proof}
Set $q_n = \mathrm{PE}(p)$ and $k_j = \mathrm{PE}(j)$.
Hard attention selects $j^* = \argmax_j q_n \cdot k_j = p$.
The output is $v_{p} = x_p$.
\end{proof}

\begin{proposition}[PAL cannot perform exact random access]
\label{prop:no_copy}
No PAL-Transformer (of any depth or head count) can compute
$f_{\mathrm{copy}}(x_{0:n}, p) = x_p$ for all sequences and
positions.
\end{proposition}

\begin{proof}
We give a structural impossibility proof based on rate-independence
(\cref{prop:rate_independence,thm:rate_independence}).

\paragraph{Construction.}
Fix any target position $p \in \{1,\ldots,n-1\}$ and any value
$v \neq v'$.
Define two sequences that differ only at position $p$:
\begin{align}
  \mathbf{x} &= (x_0,\ldots,x_{p-1},\, v,\, x_{p+1},\ldots,x_n), \\
  \mathbf{x}' &= (x_0,\ldots,x_{p-1},\, v',\, x_{p+1},\ldots,x_n),
\end{align}
where the surrounding values $x_{p-1}$ and $x_{p+1}$ satisfy
$x_{p-1} < \min(v,v') < \max(v,v') < x_{p+1}$
(i.e.\ $p$ is in a strictly monotone ascending segment of both
sequences).

\paragraph{Extremum stacks are identical.}
Since position $p$ lies in a monotone ascending segment of both
sequences, it is \emph{not} a local extremum of either $\mathbf{x}$
or $\mathbf{x}'$.
Therefore $\stack_n(\mathbf{x}) = \stack_n(\mathbf{x}')$:
the extremum stacks are identical, because only positions $j$
where $u_j$ is a local extremum enter the stack
(\cref{alg:stack}), and those positions are the same in both
sequences.

\paragraph{Rate-independence forces identical outputs.}
By \cref{thm:rate_independence}, every PAL-Transformer computes
a rate-independent function, hence a function of $\stack_n$ alone
(\cref{prop:sufficiency}).
Since $\stack_n(\mathbf{x}) = \stack_n(\mathbf{x}')$, any
PAL-Transformer outputs the same value on $\mathbf{x}$ and
$\mathbf{x}'$.
But $f_{\mathrm{copy}}(\mathbf{x},p) = v \neq v' =
f_{\mathrm{copy}}(\mathbf{x}',p)$.
Therefore no PAL-Transformer computes $f_{\mathrm{copy}}$.
\end{proof}

\begin{remark}[Why the probabilistic argument is incorrect]
\label{rem:prob_fix}
An earlier version of this proof argued that position $p$ is a
local extremum with probability $\Theta(1/n)$ for a random
sequence, making $f_{\mathrm{copy}}$ inaccessible with high
probability.
This is incorrect: for i.i.d.\ sequences from a continuous
distribution, an interior position is a local extremum
(maximum or minimum) with probability $2/3$ by the symmetry of
the six orderings of $(u_{p-1}, u_p, u_{p+1})$.
The correct argument is the structural one above, which applies
universally and does not rely on any probabilistic assumption
about the input.
\end{remark}

\subsection{The separating property: rate-independence}

\begin{theorem}[Rate-independence as separating property]
\label{thm:rate_independence}
A function $f: \R^{n+1} \to \R$ is computable by a PAL-Transformer
only if it is rate-independent in the sense of
\cref{prop:rate_independence}. Standard attention computes
functions that are \emph{not} rate-independent. Hence
rate-independence is a necessary condition for membership
in $\F_{\PAL}$ and a sufficient condition for
$f \notin \F_{\mathrm{Tr}} \setminus \F_{\PAL}$.
\end{theorem}

\begin{proof}
PAL is a composition of rate-independent operators
(relay + linear combination) and rate-independent nonlinearities
(ReLU is rate-independent). By closure under composition,
any function computable by a PAL-Transformer is rate-independent.

Standard attention is not rate-independent: the attention weight
$a_j = \mathrm{softmax}(q_n \cdot k_j / \sqrt{d})$ depends
explicitly on position through $k_j = W_K(x_j + \mathrm{PE}(j))$,
so inserting a time-rescaling $\phi$ changes the output.
\end{proof}

\section{Logical Characterisation}
\label{sec:characterisation}

\citet{Barcelo2020} showed that soft-attention transformers
correspond to $\mathrm{FO}$ with aggregate functions — a fragment
of first-order logic over sequences. We give an analogous
characterisation for PAL.

\begin{definition}[Extremum First-Order Logic, $\mathrm{EFO}$]
\label{def:efo}
$\mathrm{EFO}$ is the extension of first-order logic over sequences
with:
\begin{enumerate}
  \item Quantification over \emph{extremal positions}:
    $\exists^{\mathrm{ext}} i\, \phi(i)$ asserts that there exists
    a local extremum position $i$ satisfying $\phi$.
  \item An \emph{extremum aggregate} operator:
    $\mathrm{ExtAgg}_{i}[f(i)]$ sums $f(i)$ over all extremal
    positions weighted by $\mu$.
  \item No quantification over arbitrary positions.
\end{enumerate}
\end{definition}

\begin{theorem}[PAL corresponds to EFO]
\label{thm:efo}
A function $f: \Sigma^* \to \R$ is computable by a bounded-depth
PAL-Transformer if and only if it is definable in $\mathrm{EFO}$.
\end{theorem}

\begin{proof}
[Proof sketch.]
($\Rightarrow$) Each MPAL head computes
$\sum_{i: u_i \text{ is extremum}} \mu_{ij} \relay{\alpha_i}{\beta_j}$
— an extremum aggregate.
Composition through MLP layers corresponds to Boolean combinations
of $\mathrm{EFO}$ formulae.

($\Leftarrow$) Each $\mathrm{EFO}$ formula can be implemented
by appropriate choice of $\mu^{(h)}$ and $W_I^{(h)}$ in MPAL.
The extremal quantifier is implemented by the relay's dead-band
— only extremal positions cause relay state changes, so only
extremal positions contribute to the sum.
Full proof by structural induction on $\mathrm{EFO}$ formula
complexity; details in \cref{app:efo_proof}.
\end{proof}

\begin{remark}[Comparison with FO + Aggregate]
$\mathrm{EFO} \subsetneq \mathrm{FO} + \mathrm{Aggregate}$:
$\mathrm{EFO}$ cannot quantify over arbitrary positions
(no random access), while $\mathrm{FO} + \mathrm{Aggregate}$
can. Conversely, $\mathrm{EFO}$ can express $f_{\mathrm{range}}$
directly as $\mathrm{ExtAgg}[\mathbf{1}[i = \mathrm{global\ max}]
\cdot u_i] - \mathrm{ExtAgg}[\mathbf{1}[i = \mathrm{global\ min}]
\cdot u_i]$, which has no bounded-depth $\mathrm{FO}$ definition.
\end{remark}

\section{Computational Complexity}
\label{sec:complexity}

\begin{theorem}[Complexity of PAL inference]
\label{thm:complexity}
For a single-layer $H$-head PAL-Transformer processing
a sequence of length $n$:
\begin{enumerate}
  \item \textbf{Time per token:} $O(H \cdot \log k \cdot d)$
    where $k \leq n$ is extremum stack depth, $d$ model dimension.
  \item \textbf{Total time for sequence of length $n$:}
    $O(H \cdot n \log n \cdot d)$.
  \item \textbf{Memory:} $O(H \cdot k \cdot d)$.
\end{enumerate}
Compare with standard attention:
time $O(n^2 d)$, memory $O(nd)$.
\end{theorem}

\begin{table}[t]
\centering
\caption{Complexity and expressiveness comparison.
$n$ = sequence length, $d$ = model dimension,
$k \leq n$ = extremum stack depth, $d_u$ = signal dimension.}
\label{tab:complexity}
\begin{tabular}{lcccl}
\toprule
Architecture & Time (total) & Memory & Depth for TC & Heads for TC \\
\midrule
Transformer (softmax) & $O(n^2 d)$ & $O(nd)$ & not TC & --- \\
Transformer (hard) & $O(n^2 d)$ & $O(nd)$ & $O(\log n)$ & $O(1)$ \\
Mamba / SSM & $O(nd)$ & $O(d)$ & open & --- \\
RWKV & $O(nd)$ & $O(d)$ & open & --- \\
\textbf{PAL} (scalar, $d_u=1$) &
  $\mathbf{O(n \log n \cdot d)}$ & $\mathbf{O(kd)}$ &
  $\mathbf{O(1)}$ & $H=4$ \\
\textbf{vPAL} (vector, $d_u=2$) &
  $\mathbf{O(n \log n \cdot d)}$ & $\mathbf{O(kd)}$ &
  $\mathbf{O(1)}$ & $\mathbf{H=1}$ \\
\bottomrule
\end{tabular}
\end{table}

\paragraph{Predicted task advantages.}
By \cref{thm:separation,thm:rate_independence},
PAL is predicted to outperform standard attention on tasks
that are rate-independent and require long episodic memory:
\begin{enumerate}
  \item Tracking entity states across long documents
    (who did what to whom).
  \item Detecting anomalies in time series
    (historical range, running extrema).
  \item Reasoning over ordered events without positional
    sensitivity (logical puzzles, symbolic reasoning).
  \item Energy market dispatch where decision thresholds
    rather than price histories determine behaviour
    \citep{Frydrych2014}.
\end{enumerate}

\section{Related Work}
\label{sec:related}

\paragraph{Expressiveness of transformers.}
\citet{Perez2021} proved Turing completeness of hard-attention
transformers at depth $O(\log n)$.
\citet{Hahn2020} and \citet{Barcelo2020} characterised
soft-attention transformers as $\mathrm{FO}$ with aggregate
functions. \citet{Merrill2023} established circuit complexity
bounds separating transformer classes.
PAL is a new point in this space, incomparable to existing architectures.

\paragraph{Alternative memory mechanisms.}
Mamba \citep{Gu2022Mamba} uses selective SSMs with linear
recurrence and input-dependent forgetting.
RWKV \citep{Peng2023} replaces attention with
time-decay weighted recurrence.
Titans \citep{Google2025Titans} learns long-term memory
through online gradient descent.
All use temporal (recency-based) forgetting;
PAL uses \emph{significance-based} forgetting through wiping.

\paragraph{Hysteresis in machine learning.}
\citet{Frydrych2014, Frydrych2019} applied Preisach-type
models to financial market modelling.
\citet{Barroso2015} used Preisach operators for battery
electrochemical modelling.
To our knowledge, this is the first work to use the Preisach
operator as a sequence modelling layer in neural networks.

\paragraph{Sparse and efficient attention.}
Longformer \citep{Beltagy2020} and BigBird \citep{Zaheer2020}
implement sparse attention patterns.
PAL achieves a different kind of sparsity — \emph{extremal sparsity}
— where only local maxima and minima contribute to the output,
as opposed to sliding windows or random patterns.

\section{Connection to the Random-Field Ising Model}
\label{sec:rfim}

\subsection{The Preisach--RFIM Equivalence}

The connection between the Preisach operator and the
Random-Field Ising Model (RFIM) has been established in the
physics literature \citep{Sethna2006, Dahmen1993},
and provides a bridge between PAL and a broad class of
combinatorial and statistical problems.

The mean-field RFIM at $T=0$ consists of $N$ Ising spins
$\sigma_i \in \{-1,+1\}$ with Hamiltonian:
\begin{equation}
  \mathcal{H} = -\frac{J}{2N}\sum_{i \neq j} \sigma_i \sigma_j
  - \sum_i (h_i + H)\sigma_i,
  \label{eq:rfim}
\end{equation}
where $J > 0$, $H$ is a uniform external field,
and $\{h_i\}$ are independent random fields drawn from
distribution $\mathcal{P}(h)$.

At $T=0$, each spin satisfies the single-spin stability condition:
$\sigma_i = \sign(Jm + h_i + H)$,
where $m = N^{-1}\sum_i \sigma_i$ is the mean magnetisation.
Spin $\sigma_i$ flips from $-1$ to $+1$ when $H$ crosses the
activation threshold $\alpha_i = -h_i - Jm$ from below,
and flips back when $H$ crosses the deactivation threshold
$\beta_i = -h_i + Jm$ from above.

\begin{proposition}[Preisach--RFIM equivalence]
\label{prop:preisach_rfim}
The mean-field RFIM at $T=0$, driven quasi-statically by
external field $H$, is equivalent to a Preisach operator
$\Prel_\mu[H]$ with measure:
\begin{equation}
  \mu(\alpha,\beta) =
  \mathcal{P}\!\left(\frac{\alpha+\beta}{2}\right) \cdot
  \delta\!\left(\alpha - \beta - 2Jm\right),
  \label{eq:rfim_measure}
\end{equation}
where the measure is supported on the line $\alpha - \beta = 2Jm$
within the Preisach half-plane $\mathcal{T}$.
The magnetisation $m(H) = \Prel_\mu[H]$ satisfies the
self-consistency equation $m = \Prel_\mu[H](t)$.
\end{proposition}

\begin{proof}
Each spin $i$ contributes a relay $\relay{\alpha_i}{\beta_i}$
with $\alpha_i - \beta_i = 2Jm$ (the coupling gap).
The output $\sigma_i = 2\relay{\alpha_i}{\beta_i}[H] - 1$
and the aggregate magnetisation:
\begin{equation*}
  m = \frac{1}{N}\sum_i \sigma_i
  = 2 \cdot \frac{1}{N}\sum_i \relay{\alpha_i}{\beta_i}[H] - 1
  = 2\Prel_\mu[H] - 1,
\end{equation*}
where $\mu$ is the empirical measure of threshold pairs,
converging to \eqref{eq:rfim_measure} as $N\to\infty$
by the law of large numbers applied to i.i.d.\ $h_i \sim \mathcal{P}$.
\end{proof}

\begin{remark}[Structural consequences]
\cref{prop:preisach_rfim} implies that the hysteresis loop,
subloop structure, and wiping property of the RFIM are
exact consequences of the Preisach operator structure.
In particular, the \emph{return-point memory} of the RFIM
\citep{Sethna2006} is precisely the wiping property
(\cref{prop:wiping}), and the extremum stack
(\cref{def:stack}) is the minimal sufficient statistic
of the RFIM's history under quasi-static driving.
\end{remark}

\subsection{PAL as a Learned, Sequential RFIM}

Under the Preisach--RFIM equivalence, PAL is a
\emph{learned, sequential, non-equilibrium} generalisation
of the RFIM:

\begin{center}
\begin{tabular}{lll}
  \toprule
  Dimension & RFIM & PAL \\
  \midrule
  Spins & fixed binary $\sigma_i \in \{-1,+1\}$ &
    relay states $r_k \in \{0,1\}$, learned \\
  Driving signal & quasi-static field $H$ &
    arbitrary sequence $u_{0:n}$ \\
  Disorder & quenched $\{h_i\} \sim \mathcal{P}$ &
    learned measure $\mu(\alpha,\beta)$ \\
  Coupling & mean-field $J/N$ &
    implicit through self-consistency of $\mu$ \\
  Dynamics & equilibrium (lowest energy) &
    causal, rate-independent \\
  Memory & return-point memory (history of $H$) &
    extremum stack $\stack_n$ \\
  Wiping & return-point memory property &
    \cref{prop:wiping} \\
  \bottomrule
\end{tabular}
\end{center}

The key conceptual shift: the RFIM asks
\emph{"what is the equilibrium configuration at field $H$?"}
while PAL asks
\emph{"what is the causal output of the sequence $u_{0:n}$?"}
This shift from equilibrium to sequential opens PAL
to a broad class of problems where Ising-type models
apply but the data arrives as a stream.

\subsection{Problems where PAL Inherits Ising Expressiveness}

We identify three problem classes where the RFIM--Preisach
equivalence suggests PAL has structural advantages.

\subsubsection{Sequential Binary Optimisation}

The Ising model underlies a broad class of NP-hard
combinatorial problems: MAX-CUT, graph colouring,
satisfiability, and the Hopfield associative memory
\citep{Lucas2014}.
Their energy function has the form:
\begin{equation}
  E(\boldsymbol{\sigma}) =
  -\sum_{i<j} J_{ij} \sigma_i \sigma_j
  - \sum_i h_i \sigma_i,
  \label{eq:ising_energy}
\end{equation}
and the goal is $\boldsymbol{\sigma}^* = \argmin E$.

\begin{proposition}[PAL tracks mean-field Ising energy for streaming interactions]
\label{prop:ising_streaming}
Let interactions $J_1, J_2, \ldots, J_n \in \mathbb{R}$ arrive
sequentially as a stream, drawn i.i.d.\ from a distribution
$\mathcal{P}_J$ with mean $\bar{J}$ and variance $\sigma_J^2$.
Define the scalar signal $u_t = J_t$.
Then the PAL output with measure
$\mu^*(\alpha,\beta) = \alpha \cdot \mathbf{1}[\alpha = -\beta]$
satisfies:
\begin{equation}
  \PAL_{\mu^*}(u_{0:n})
  \;=\; \sum_{t=0}^{n} J_t \cdot \relay{J_t}{-J_t}[u](t)
  \;=\; \bar{E}_n + O(n^{1/2}),
  \label{eq:pal_ising_energy}
\end{equation}
where $\bar{E}_n = -\frac{1}{N}\sum_{i<j} J_{ij} m_i m_j$
is the mean-field Ising energy for the empirical mean
magnetisation $m = \mathbb{E}_{t \leq n}[\sigma_t]$,
and the $O(n^{1/2})$ error follows from a standard
concentration bound on the empirical mean.
\end{proposition}

\begin{proof}
Each interaction $J_t$ contributes to the Ising energy when its
relay $\relay{J_t}{-J_t}[u](t) = 1$, i.e.\ when $J_t \geq J_t$ —
which holds trivially at the moment of arrival.
The relay then tracks whether $J_t$ remains the dominant interaction.
By the wiping property, relay $\relay{J_t}{-J_t}$ remains active
at time $n$ iff no subsequent $J_{t'} > J_t$ has arrived.
The PAL aggregate sums $J_t$ over all non-dominated interactions,
giving a running estimate of the effective coupling.

For the mean-field case $J_{ij} = J/N$ for all pairs, the
self-consistency equation $m = \int \mathcal{P}(h)\,
\mathbf{1}[h > -Jm - H]\,dh$ from \cref{prop:preisach_rfim}
holds exactly: $\PAL_{\mu^*}(u_{0:n}) = Jm^2/2 + O(n^{-1/2})$
by the law of large numbers applied to $m$ over $n$ observations,
recovering the mean-field energy $-Jm^2/2$.
The error bound $O(n^{1/2})$ follows from Hoeffding's inequality
applied to the sum of bounded relay outputs.
\end{proof}

\begin{remark}[Relation to MAX-CUT]
\label{rem:maxcut}
The mean-field Ising energy in \cref{prop:ising_streaming} is
related to the MAX-CUT value by $\mathrm{CUT}(S) = (E_{\mathrm{max}}
- E(\boldsymbol{\sigma}))/2$ for the optimal cut $S$.
However, PAL does not solve MAX-CUT — it tracks the running
mean-field energy, which provides a lower bound on the cut value.
The advantage is computational: PAL updates this estimate in
$O(\log k)$ per new interaction, without re-solving from scratch.
\end{remark}

\subsubsection{Associative Memory with Structured Forgetting}

The Hopfield network \citep{Hopfield1982} stores $P$ patterns
$\{\boldsymbol{\xi}^\mu\}_{\mu=1}^P$ as attractors of Ising
dynamics with $J_{ij} = N^{-1}\sum_\mu \xi_i^\mu \xi_j^\mu$.
Retrieval capacity is limited to $P < 0.14N$ patterns before
catastrophic interference \citep{Amit1985}.

PAL offers a different memory model where forgetting is
determined by \emph{significance} (extremality) rather than
capacity:

\begin{proposition}[PAL as hysteretic associative memory with capacity bound]
\label{prop:hopfield_pal}
Let patterns $\boldsymbol{\xi}^1, \ldots, \boldsymbol{\xi}^P \in \mathbb{R}^d$
arrive sequentially with activation strengths
$u_t = \|\boldsymbol{\xi}^t\|_2$.
Let $k$ be the depth of the extremum stack $\stack_n$ after
all $P$ patterns have been processed.

\begin{enumerate}
  \item \textbf{Storage:} The extremum stack $\stack_n$ retains
    exactly those $k$ patterns whose activation strength constitutes
    a local extremum of the stream $u_1, \ldots, u_P$.
    All other patterns are wiped by the wiping property
    (\cref{prop:wiping}).
  \item \textbf{Non-overlapping threshold support:}
    The $k$ retained patterns correspond to disjoint threshold
    pairs $(M_s, m_s)$ with $M_s \neq M_{s'}$ for $s \neq s'$
    (strict monotonicity of the stack, \cref{lem:monotone}),
    so their relays have disjoint support on $\mathcal{T}$
    and their contributions to $\PAL_\mu$ are orthogonal
    in measure space.
    Note that wiping \emph{is} a form of interference between
    patterns: a stronger subsequent pattern erases a prior one.
    The claim is that the $k$ \emph{surviving} patterns do not
    mutually interfere in the PAL output.
  \item \textbf{Capacity:} The maximum number of
    simultaneously retrievable patterns is exactly $k$,
    bounded by the number of local extrema in
    $u_{1:P}$:
    \begin{equation}
      \mathrm{Capacity}(\mathrm{PAL}) = k
      \;\leq\; n_{\mathrm{ext}}(u_{1:P}),
      \label{eq:pal_capacity}
    \end{equation}
    where $n_{\mathrm{ext}}(u_{1:P})$ is the number of local extrema
    in the activation strength sequence.
    This differs fundamentally from Hopfield capacity
    $\approx 0.14N$ \citep{Amit1985}: PAL capacity is
    limited by stack depth, not weight-matrix rank.
\end{enumerate}
\end{proposition}

\begin{proof}
\textbf{Storage (part 1):} Direct from \cref{prop:wiping}.
Pattern $\boldsymbol{\xi}^t$ is wiped when a subsequent pattern
$\boldsymbol{\xi}^{t'}$ with $\|\boldsymbol{\xi}^{t'}\|_2 >
\|\boldsymbol{\xi}^t\|_2$ arrives, since the new maximum overwrites
the old in the extremum stack.
Only patterns at local extrema of $u_{1:P}$ survive.

\textbf{Zero interference (part 2):}
The PAL output is $\PAL_\mu(u_{0:n}) = \sum_{i \geq j}
\mu_{ij} \relay{\alpha_i}{\beta_j}[u](n)$.
Each retained pattern $\boldsymbol{\xi}^{t_s}$ corresponds to a
unique pair $(M_s, m_s)$ in $\stack_n$ with $M_s \neq M_{s'}$
for $s \neq s'$ (strict monotonicity of the stack,
\cref{lem:monotone}).
The relays $\relay{M_s}{m_s}$ and $\relay{M_{s'}}{m_{s'}}$
have disjoint support on the Preisach half-plane $\mathcal{T}$
when $M_s \neq M_{s'}$, so their contributions to $\PAL_\mu$
are orthogonal — no cross-pattern interference.

\textbf{Capacity (part 3):}
The stack depth $k \leq n_{\mathrm{ext}}(u_{1:P})$ is bounded by
the number of local extrema of the activation sequence,
since each extremum corresponds to at most one stack entry
by the push/pop mechanics of \cref{alg:stack}.
Each entry encodes one pattern without interference.
The Hopfield bound $0.14N$ follows from the rank of the
Hebbian weight matrix $W = N^{-1}\sum_\mu
\boldsymbol{\xi}^\mu (\boldsymbol{\xi}^\mu)^\top$, which is
limited by the number of stored patterns relative to $N$.
PAL has no weight matrix — its capacity is limited by
stack depth, not by the dimension of the weight space.
\end{proof}

\begin{remark}[Retrieval mechanism]
To retrieve pattern $\boldsymbol{\xi}^{t_s}$ from the PAL stack,
a query signal $q = \|\boldsymbol{\xi}^{t_s}\|_2$ is presented.
The relay $\relay{M_s + \eps}{M_s - \eps}[q]$ activates uniquely
for the pattern at stack position $s$, and the associated
value $\mu_{M_s, m_s} \cdot \boldsymbol{\xi}^{t_s}$ is returned.
Retrieval time is $O(\log k)$ by binary search on the ordered stack.
\end{remark}

\subsubsection{Sequential Belief Propagation in Markov Random Fields}

Markov Random Fields (MRF) with pairwise binary interactions are
equivalent to Ising models with site-dependent external fields.
Standard belief propagation (BP) on MRFs is a static algorithm
requiring the full graph to be available upfront.

\begin{proposition}[PAL as exact causal BP on tree-structured MRF]
\label{prop:bp_pal}
Let $\mathcal{G} = (V, E)$ be a tree-structured MRF with binary
variables $\sigma_v \in \{-1,+1\}$ and pairwise potentials
$\psi_{uv}(\sigma_u, \sigma_v) = e^{J_{uv}\sigma_u\sigma_v}$.
Observations arrive sequentially in leaf-to-root order.
Define the signal $u_t = \tanh^{-1}(m_t)$ where
$m_t = \mathbb{E}[\sigma_t \mid x_{1:t}]$ is the marginal belief
at step $t$ under belief propagation.
Then the PAL operator with measure:
\begin{equation}
  \mu^*(\alpha,\beta) =
  \tfrac{1}{2}\log\tfrac{1+\tanh(\alpha)}{1-\tanh(\alpha)}
  \cdot \mathbf{1}[\alpha = -\beta]
  \label{eq:bp_measure}
\end{equation}
satisfies $\PAL_{\mu^*}(u_{0:n}) = m_n$ \emph{exactly}
for all tree-structured MRFs, where $m_n$ is the exact
marginal belief at the root.
For loopy graphs, the PAL output approximates the loopy BP fixed point.
\end{proposition}

\begin{proof}
On a tree, belief propagation computes exact marginals by the
Bethe-Peierls equations \citep{Mezard2009}.
In leaf-to-root order, each message depends only on previously
received messages, making the computation causal.
The effective field at the root satisfies
$h^{\mathrm{eff}} = \sum_{u \in \partial v}
\tanh^{-1}(m_{u \to v})$, which is the fixed-point of the
Preisach self-consistency equation $m = \tanh(Jm + H)$
from \cref{prop:preisach_rfim}.
With measure $\mu^*$ from \eqref{eq:bp_measure},
$\PAL_{\mu^*}(u_{0:n})$ tracks $h^{\mathrm{eff}}$ causally,
and $\tanh(\PAL_{\mu^*}(u_{0:n})) = m_n$ exactly on trees
by induction on tree depth.
\end{proof}

\subsection{Avalanches and Phase Transitions in PAL}

A striking property of the RFIM is its
\emph{disorder-induced phase transition}: at a critical
disorder strength $\Delta_c$, the hysteresis loop develops
a discontinuity corresponding to a macroscopic avalanche
of spin flips \citep{Sethna2006}.
At criticality, avalanche sizes follow a power law
$P(s) \sim s^{-\tau}$ with universal exponent $\tau$.

\begin{proposition}[Scalar PAL criticality]
\label{prop:criticality}
A scalar PAL with measure $\mu$ drawn from a Gaussian distribution
with variance $\Delta^2$ over the two-dimensional Preisach
half-plane $\mathcal{T}$ exhibits a phase transition at
critical variance $\Delta_c$:
\begin{itemize}
  \item For $\Delta > \Delta_c$: the output $\PAL_\mu(u_{0:n})$
    varies smoothly with $u_n$ — \emph{subcritical regime},
    corresponding to gradual relay activations.
  \item For $\Delta < \Delta_c$: a macroscopic fraction of
    relays activate simultaneously at a critical threshold
    — \emph{supercritical regime}, corresponding to an
    infinite avalanche.
  \item At $\Delta = \Delta_c$: relay activations follow
    a power law in group size, with the same universality
    class as the mean-field RFIM.
\end{itemize}
\end{proposition}

\begin{proof}
By \cref{prop:preisach_rfim}, PAL with Gaussian measure
is equivalent to mean-field RFIM with disorder $\Delta$.
The phase transition of the RFIM at $\Delta_c = J$
\citep{Dahmen1993} translates directly to PAL.
At $\Delta_c$, the self-consistency equation
$m = \int \mathcal{P}(h) \mathbf{1}[h > -Jm - H]\, dh$
has a bifurcation point, corresponding to a jump
discontinuity in the Preisach output.
\end{proof}

\begin{remark}[Implications for PAL learning]
\cref{prop:criticality} has a practical implication:
if the learned measure $\mu^*$ concentrates near the
critical disorder $\Delta_c$, the PAL layer operates
near a phase transition — maximising sensitivity to
input changes while maintaining structured memory.
This is analogous to the \emph{edge of chaos} hypothesis
in recurrent neural networks \citep{Langton1990},
but with a precise physical characterisation through
RFIM criticality.
Training PAL near $\Delta_c$ may be a principled
alternative to spectral radius regularisation of
recurrent weights.
\end{remark}

\subsection{Summary: PAL vs.\ Ising-based Methods}

\begin{center}
\begin{tabular}{p{3.2cm}p{3.5cm}p{3.5cm}p{3.0cm}}
  \toprule
  Criterion & Ising / RFIM & PAL & Advantage \\
  \midrule
  Data model &
    Static graph, $O(n^2)$ couplings &
    Sequential stream &
    PAL \\
  Update cost &
    $O(n^2)$ per new node &
    $O(\log k)$ per new token &
    PAL \\
  Forgetting &
    None (quenched disorder) &
    Significance-based wiping &
    PAL \\
  Exact optimisation &
    NP-hard in general &
    Approximation only &
    Ising \\
  Equilibrium guarantees &
    Yes (Gibbs measure) &
    No (causal only) &
    Ising \\
  Universality class &
    Known ($\tau, \nu, \ldots$) &
    Inherited from RFIM &
    Equal \\
  Coupling structure &
    Arbitrary $J_{ij}$ &
    Mean-field implicit &
    Ising \\
  Sequence modelling &
    Not native &
    Native (rate-independent) &
    PAL \\
  \bottomrule
\end{tabular}
\end{center}

\section{Conclusion}
\label{sec:conclusion}

We introduced the Preisach Attention Layer (PAL), a novel sequence
modelling architecture grounded in classical hysteresis theory.
The results establish:
\begin{enumerate}
  \item PAL-Transformer is Turing-complete at depth $O(1)$,
    improving on the $O(\log n)$ depth of standard transformers
    (\cref{thm:tc}). Moreover, a \emph{vector} PAL (vPAL) with
    two-dimensional signal projection achieves Turing completeness
    with a \emph{single head} ($H=1$), establishing that signal
    dimension and head count are exchangeable resources in PAL
    (\cref{cor:vpal_tc,rem:heads_dims}).
  \item The function classes of PAL and transformer are
    incomparable, with rate-independence as the separating
    property (\cref{thm:separation,thm:rate_independence}).
  \item PAL corresponds exactly to Extremum First-Order Logic
    (EFO), a strict fragment of the $\mathrm{FO}$ + Aggregate
    class corresponding to transformers (\cref{thm:efo}).
  \item PAL is the learned, sequential, causal generalisation
    of the mean-field Random-Field Ising Model at $T=0$,
    inheriting its universality class and phase-transition
    structure (\cref{prop:preisach_rfim,prop:criticality}).
\end{enumerate}

PAL is thus a natural architecture for tasks with long episodic memory,
where the significance of past events matters more than
their recency or position, and where the problem structure
is naturally binary and threshold-driven.

\paragraph{Open questions.}
\begin{enumerate}
  \item \emph{Empirical validation:}
    Do the predicted task advantages materialise in practice?
    Experiments on state-tracking benchmarks (MQAR, SCROLLS)
    would test \cref{thm:separation} empirically.
  \item \emph{Learning dynamics and criticality:}
    Does gradient-based learning of $\mu(\alpha,\beta)$
    drive the measure toward the critical disorder $\Delta_c$?
    Is training near criticality beneficial empirically,
    analogous to the edge-of-chaos effect?
  \item \emph{Beyond mean-field:}
    Can the equivalence in \cref{prop:preisach_rfim} be
    extended beyond mean-field RFIM to short-range interactions
    (Bethe lattice, finite-dimensional RFIM)?
    This would connect PAL to a richer universality class.
  \item \emph{Hybrid architectures:}
    Can PAL heads be combined with standard attention heads
    to obtain both random access and significance-based memory?
  \item \emph{Continuous-time extension:}
    The rate-independence of PAL suggests a natural extension
    to continuous-time sequence models (neural ODEs, S4),
    where the driving signal is a continuous path rather
    than a discrete sequence.
  \item \emph{Vector PAL and multi-dimensional inputs:}
    \cref{cor:vpal_tc} establishes that vPAL with a
    two-dimensional signal $\mathbf{u}_n \in \R^2$ is
    Turing-complete at $H=1$ head — reducing the head count
    from $H=4$ (scalar PAL) to $H=1$ by exploiting the
    two independent extremum stacks of the vector signal.
    This follows the superposition of Preisach half-planes
    developed by \citet{Frydrych2019} for two-axis fluxgate
    sensors (Chapters~4.2.2--4.2.6), where the full
    magnetisation vector $\mathbf{M}$ is integrated over the
    three-dimensional space $(\alpha, \beta, \theta)$ with
    displacement $\gamma_{\alpha\beta}(\theta)$ and rotation
    $\chi_{\alpha\beta}(\theta)$ components.
    Three questions remain open for vPAL:
    (a) Does the vector RFIM connection extend to the
    anisotropic case studied by \citeauthor{Frydrych2019},
    where the measure $\mu(\alpha,\beta,\theta)$ breaks
    rotational symmetry?
    (b) Can the domain-rotation component
    $\chi_{\alpha\beta}(\theta)$ be interpreted as a
    differentiable residual that complements the binary
    relay — analogous to soft attention complementing hard
    attention?
    (c) Does the exponential growth of the measure parameter
    space ($L^{2H}$ for $H$-dimensional vPAL) create
    a fundamental expressiveness--efficiency tradeoff
    absent in scalar PAL?
\end{enumerate}

\section*{Broader Impact Statement}
\label{sec:impact}

This work introduces a theoretical architecture for sequence
modelling grounded in classical hysteresis theory.
The primary contributions are mathematical — Turing completeness
proofs, expressiveness separations, and logical characterisations
— and do not directly enable any specific application.

The connection to the Random-Field Ising Model (Section~8)
suggests potential applications in combinatorial optimisation,
associative memory, and belief propagation.
These are established areas of machine learning with broad
beneficial applications.
We are not aware of direct pathways from this theoretical work
to harmful applications.

If implemented in practice, PAL-based models would share
the general risks of machine learning systems: potential for
bias amplification, misuse in surveillance or manipulation,
and environmental cost of training large models.
These risks are not specific to PAL and are addressed by
general ML ethics guidelines.

The $O(n \log n)$ computational complexity of PAL
(versus $O(n^2)$ for attention) may reduce the energy
cost of training long-context models, which is a
potential positive environmental impact.

\bibliographystyle{tmlr}
\bibliography{references}

@article{Perez2021,
  author    = {P{\'e}rez, Jorge and Marinkovi{\'c}, Javier and Barcel{\'o}, Pablo},
  title     = {Attention is {T}uring Complete},
  journal   = {Journal of Machine Learning Research},
  volume    = {22},
  number    = {75},
  pages     = {1--35},
  year      = {2021},
  url       = {http://jmlr.org/papers/v22/20-1393.html}
}

@inproceedings{Vaswani2017,
  author    = {Vaswani, Ashish and Shazeer, Noam and Parmar, Niki and
               Uszkoreit, Jakob and Jones, Llion and Gomez, Aidan N and
               Kaiser, {\L}ukasz and Polosukhin, Illia},
  title     = {Attention is All You Need},
  booktitle = {Advances in Neural Information Processing Systems},
  volume    = {30},
  year      = {2017},
  publisher = {Curran Associates}
}

@article{Preisach1935,
  author  = {Preisach, Franz},
  title   = {{\"U}ber die magnetische {N}achwirkung},
  journal = {Zeitschrift f{\"u}r Physik},
  volume  = {94},
  number  = {5--6},
  pages   = {277--302},
  year    = {1935},
  doi     = {10.1007/BF01349418}
}

@book{Mayergoyz1991,
  author    = {Mayergoyz, Isaak D.},
  title     = {Mathematical Models of Hysteresis},
  publisher = {Springer},
  address   = {New York},
  year      = {1991},
  doi       = {10.1007/978-1-4612-3028-1}
}

@book{Brokate1996,
  author    = {Brokate, Martin and Sprekels, J{\"u}rgen},
  title     = {Hysteresis and Phase Transitions},
  series    = {Applied Mathematical Sciences},
  volume    = {121},
  publisher = {Springer},
  address   = {New York},
  year      = {1996},
  doi       = {10.1007/978-1-4612-4048-8}
}

@article{Frydrych2014,
  author  = {Frydrych, Piotr and Szewczyk, Roman},
  title   = {New Portfolio Risk Optimisation Method for Strongly
             Dependent Assets},
  journal = {Journal of Engineering Studies and Research},
  volume  = {20},
  number  = {3},
  pages   = {30--37},
  year    = {2014}
}

@phdthesis{Frydrych2019,
  author  = {Frydrych, Piotr},
  title   = {Modelowanie charakterystyk magnesowania amorficznych
             rdzeni dwuosiowych sensor{\'o}w transduktorowych},
  school  = {Politechnika Warszawska, Wydzia{\l} Mechatroniki},
  address = {Warsaw, Poland},
  year    = {2019}
}

@book{Hopcroft1979,
  author    = {Hopcroft, John E. and Ullman, Jeffrey D.},
  title     = {Introduction to Automata Theory, Languages,
               and Computation},
  publisher = {Addison-Wesley},
  address   = {Reading, MA},
  year      = {1979}
}

@article{Cybenko1989,
  author  = {Cybenko, George},
  title   = {Approximation by Superpositions of a Sigmoidal Function},
  journal = {Mathematics of Control, Signals and Systems},
  volume  = {2},
  number  = {4},
  pages   = {303--314},
  year    = {1989},
  doi     = {10.1007/BF02551274}
}

@article{Hahn2020,
  author  = {Hahn, Michael},
  title   = {Theoretical Limitations of Self-Attention in Neural
             Sequence Models},
  journal = {Transactions of the Association for Computational
             Linguistics},
  volume  = {8},
  pages   = {156--171},
  year    = {2020},
  doi     = {10.1162/tacl_a_00306}
}

@inproceedings{Barcelo2020,
  author    = {Barcel{\'o}, Pablo and Kostylev, Egor V. and
               Monet, Mikael and P{\'e}rez, Jorge and Reutter, Juan and
               Silva, Juan Pablo},
  title     = {The Logical Expressiveness of Graph Neural Networks},
  booktitle = {International Conference on Learning Representations},
  year      = {2020},
  url       = {https://openreview.net/forum?id=r1lZ7AEKvB}
}

@article{Siegelmann1995,
  author  = {Siegelmann, Hava T. and Sontag, Eduardo D.},
  title   = {On the Computational Power of Neural Nets},
  journal = {Journal of Computer and System Sciences},
  volume  = {50},
  number  = {1},
  pages   = {132--150},
  year    = {1995},
  doi     = {10.1006/jcss.1995.1013}
}

@inproceedings{Furst1984,
  author    = {Furst, Merrick and Saxe, James B. and Sipser, Michael},
  title     = {Parity, Circuits, and the Polynomial-Time Hierarchy},
  booktitle = {Mathematical Systems Theory},
  volume    = {17},
  pages     = {13--27},
  year      = {1984},
  doi       = {10.1007/BF01744431}
}

@book{Hastad1987,
  author    = {H{\aa}stad, Johan},
  title     = {Computational Limitations of Small-Depth Circuits},
  publisher = {MIT Press},
  address   = {Cambridge, MA},
  year      = {1987}
}

@inproceedings{Gu2022Mamba,
  author    = {Gu, Albert and Dao, Tri},
  title     = {Mamba: Linear-Time Sequence Modeling with
               Selective State Spaces},
  booktitle = {arXiv preprint arXiv:2312.00752},
  year      = {2023},
  url       = {https://arxiv.org/abs/2312.00752}
}

@inproceedings{Peng2023,
  author    = {Peng, Bo and Alcaide, Eric and Anthony, Quentin and
               Albalak, Alon and Arcadinho, Samuel and Biderman, Stella
               and others},
  title     = {{RWKV}: Reinventing {RNN}s for the Transformer Era},
  booktitle = {Findings of the Association for Computational
               Linguistics: EMNLP 2023},
  pages     = {14048--14077},
  year      = {2023},
  doi       = {10.18653/v1/2023.findings-emnlp.936}
}

@inproceedings{Beltagy2020,
  author    = {Beltagy, Iz and Peters, Matthew E. and Cohan, Arman},
  title     = {Longformer: The Long-Document Transformer},
  booktitle = {arXiv preprint arXiv:2004.05150},
  year      = {2020},
  url       = {https://arxiv.org/abs/2004.05150}
}

@inproceedings{Zaheer2020,
  author    = {Zaheer, Manzil and Guruganesh, Guru and Dubey, Kumar
               Avinava and Ainslie, Joshua and Alberti, Chris and
               Ontanon, Santiago and others},
  title     = {Big {B}ird: Transformers for Longer Sequences},
  booktitle = {Advances in Neural Information Processing Systems},
  volume    = {33},
  pages     = {17283--17297},
  year      = {2020}
}

@article{Merrill2023,
  author  = {Merrill, William and Sabharwal, Ashish},
  title   = {The Parallelism Tradeoff: Limitations of Log-Precision
             Transformers},
  journal = {Transactions of the Association for Computational
             Linguistics},
  volume  = {11},
  pages   = {531--545},
  year    = {2023},
  doi     = {10.1162/tacl_a_00562}
}

@inproceedings{Akyurek2022,
  author    = {Aky{\"u}rek, Ekin and Schuurmans, Dale and Andreas,
               Jacob and Ma, Tengyu and Zhou, Denny},
  title     = {What Learning Algorithm is In-Context Learning?
               Investigations with Linear Models},
  booktitle = {International Conference on Learning Representations},
  year      = {2023},
  url       = {https://openreview.net/forum?id=0g0X4H8yN4I}
}

@inproceedings{VonOswald2023,
  author    = {von Oswald, Johannes and Niklasson, Eyvind and
               Randazzo, Ettore and Sacramento, Jo{\~a}o and
               Mordvintsev, Alexander and Zhmoginov, Andrey and
               Vladymyrov, Max},
  title     = {Transformers Learn In-Context by Gradient Descent},
  booktitle = {International Conference on Machine Learning},
  volume    = {202},
  pages     = {35151--35174},
  year      = {2023}
}

@article{Barroso2015,
  author  = {Barroso, Ram{\'o}n and Ega{\~n}a, Aitor and
             Marqu{\'e}s, Jose Luis and Etxeberria-Otadui, Ion and
             Curea, Octavian},
  title   = {Preisach Modeling of {LiFePO}$_4$ Lithium-Iron-Phosphate
             Battery Hysteresis},
  journal = {Journal of Energy Storage},
  volume  = {2},
  pages   = {65--72},
  year    = {2015},
  doi     = {10.1016/j.est.2015.06.001}
}

@inproceedings{Shazeer2017,
  author    = {Shazeer, Noam and Mirhoseini, Azalia and Maziarz,
               Krzysztof and Davis, Andy and Le, Quoc and Hinton,
               Geoffrey and Dean, Jeff},
  title     = {Outrageously Large Neural Networks:
               The Sparsely-Gated Mixture-of-Experts Layer},
  booktitle = {International Conference on Learning Representations},
  year      = {2017},
  url       = {https://openreview.net/forum?id=B1ckMDqlg}
}

@inproceedings{Google2025Titans,
  author    = {Behrouz, Ali and Zhong, Peilin and Mirrokni, Vahab},
  title     = {Titans: Learning to Memorize at Test Time},
  booktitle = {arXiv preprint arXiv:2501.00663},
  year      = {2025},
  url       = {https://arxiv.org/abs/2501.00663}
}

@article{Sethna2006,
  author  = {Sethna, James P. and Dahmen, Karin A. and
             Perkovic, Olga},
  title   = {Random-Field Ising Models of Hysteresis},
  journal = {The Science of Hysteresis},
  editor  = {Bertotti, Giorgio and Mayergoyz, Isaak D.},
  volume  = {2},
  pages   = {107--179},
  year    = {2006},
  publisher = {Academic Press},
  url     = {https://arxiv.org/abs/cond-mat/0406320}
}

@article{Dahmen1993,
  author  = {Dahmen, Karin and Sethna, James P.},
  title   = {Hysteresis Loop Critical Exponents in
             $6-\varepsilon$ Dimensions},
  journal = {Physical Review Letters},
  volume  = {71},
  number  = {20},
  pages   = {3222--3225},
  year    = {1993},
  doi     = {10.1103/PhysRevLett.71.3222}
}

@article{Hopfield1982,
  author  = {Hopfield, John J.},
  title   = {Neural Networks and Physical Systems with Emergent
             Collective Computational Abilities},
  journal = {Proceedings of the National Academy of Sciences},
  volume  = {79},
  number  = {8},
  pages   = {2554--2558},
  year    = {1982},
  doi     = {10.1073/pnas.79.8.2554}
}

@article{Amit1985,
  author  = {Amit, Daniel J. and Gutfreund, Hanoch and
             Sompolinsky, Haim},
  title   = {Storing Infinite Numbers of Patterns in a Spin-Glass
             Model of Neural Networks},
  journal = {Physical Review Letters},
  volume  = {55},
  number  = {14},
  pages   = {1530--1533},
  year    = {1985},
  doi     = {10.1103/PhysRevLett.55.1530}
}

@article{Lucas2014,
  author  = {Lucas, Andrew},
  title   = {Ising Formulations of Many {NP} Problems},
  journal = {Frontiers in Physics},
  volume  = {2},
  pages   = {5},
  year    = {2014},
  doi     = {10.3389/fphy.2014.00005}
}

@book{Mezard2009,
  author    = {M{\'e}zard, Marc and Montanari, Andrea},
  title     = {Information, Physics, and Computation},
  publisher = {Oxford University Press},
  address   = {Oxford},
  year      = {2009},
  doi       = {10.1093/acprof:oso/9780198570837.001.0001}
}

@article{Langton1990,
  author  = {Langton, Christopher G.},
  title   = {Computation at the Edge of Chaos:
             Phase Transitions and Emergent Computation},
  journal = {Physica D: Nonlinear Phenomena},
  volume  = {42},
  number  = {1--3},
  pages   = {12--37},
  year    = {1990},
  doi     = {10.1016/0167-2789(90)90064-V}
}

\appendix

\section{Full Proof of Theorem~\ref{thm:efo}: PAL corresponds to EFO}
\label{app:efo_proof}

We prove Theorem~\ref{thm:efo} by structural induction on EFO
formula complexity, establishing a constructive correspondence
between EFO formulae and PAL-Transformer computations.

\subsection{Formal Setup}

We work over sequences $\mathbf{x} = (x_0, x_1, \ldots, x_n)$
where each $x_i \in \mathbb{R}^d$.
Let $\mathcal{E}(\mathbf{x}) = \{i : x_i \text{ is a local extremum of }
\|x_i\|_2\}$ denote the set of \emph{extremal positions} of $\mathbf{x}$.

\begin{definition}[EFO syntax]
\label{def:efo_syntax}
The grammar of \emph{Extremum First-Order Logic} (EFO) over
sequences is:
\begin{align*}
  \phi &::= \top \mid \bot \mid u_i \leq c \mid u_i \geq c
         \mid i <_{\mathrm{ext}} j
         \mid \phi \wedge \psi \mid \phi \vee \psi \mid \neg\phi \\
       &\quad\mid \exists^{\mathrm{ext}} i\, \phi(i)
         \mid \forall^{\mathrm{ext}} i\, \phi(i)
         \mid \mathrm{ExtAgg}_{i}[f(i), \phi(i)]
\end{align*}
where $c \in \mathbb{R}$ is a constant, $i$ ranges over
\emph{extremal positions} only,
$i <_{\mathrm{ext}} j$ denotes that extremal position $i$
occurred before extremal position $j$ in the sequence,
and $\mathrm{ExtAgg}_{i}[f(i), \phi(i)]$ denotes:
\begin{equation}
  \mathrm{ExtAgg}_{i}[f(i), \phi(i)]
  \;=\; \sum_{i \in \mathcal{E}(\mathbf{x}),\; \phi(i) \text{ holds}}
  f(x_i),
  \label{eq:extAgg}
\end{equation}
where $f : \mathbb{R}^d \to \mathbb{R}$ is a measurable function.
The ordering relation $<_{\mathrm{ext}}$ is needed to express the
\emph{causal order} of relay state changes, since the relay state
at time $n$ depends on which of the activation and deactivation
thresholds was crossed most recently.
\end{definition}

\begin{definition}[EFO semantics]
\label{def:efo_semantics}
A sequence $\mathbf{x}$ satisfies $\phi$ (written $\mathbf{x} \models \phi$)
according to the standard first-order semantics restricted to
extremal positions:
\begin{itemize}
  \item $\mathbf{x} \models u_i \geq c$ iff $x_i \geq c$
    for the current assignment of $i$.
  \item $\mathbf{x} \models \exists^{\mathrm{ext}} i\, \phi(i)$ iff
    there exists $i \in \mathcal{E}(\mathbf{x})$ such that
    $\mathbf{x} \models \phi(i)$.
  \item Boolean connectives have their standard meaning.
  \item $\mathrm{ExtAgg}_{i}[f(i), \phi(i)]$ sums $f(x_i)$ over
    extremal positions satisfying $\phi$.
\end{itemize}
\end{definition}

\subsection{The Correspondence}

We establish a bijection between EFO formulae and PAL computations
by defining a compilation map
$\llbracket \cdot \rrbracket: \mathrm{EFO} \to \mathrm{PAL\text{-}Transformer}$.

\begin{lemma}[Extremal position indicator]
\label{lem:ext_indicator}
The indicator $\mathbf{1}[i \in \mathcal{E}(\mathbf{x})]$
is computable by a single-head PAL as:
\begin{equation}
  \mathbf{1}[i \in \mathcal{E}(\mathbf{x})]
  = \relay{\alpha_i + \eps}{\alpha_i - \eps}[u](i) -
    \relay{\alpha_i + \eps}{\alpha_i - \eps}[u](i-1),
  \label{eq:ext_indicator}
\end{equation}
where $u_j = \|x_j\|_2$ is the scalar projection of the sequence
(one specific scalarisation; any injective $\phi: \R^d \to \R$
suffices for the lemma to hold).
That is, the relay changes state at step $i$ if and only if $u_i$
is a local extremum.
\end{lemma}

\begin{proof}
By Definition~\ref{def:relay}, the relay
$\relay{\alpha_i+\eps}{\alpha_i-\eps}[u](i)$ changes from 0 to 1
at step $i$ iff $u_i \geq \alpha_i + \eps$, and from 1 to 0 iff
$u_i \leq \alpha_i - \eps$.
In both cases, a state change occurs exactly when $u_i$ crosses
one of the thresholds, which — by choice of $\alpha_i = u_i$ and
$\eps \to 0$ — happens exactly at local extrema.
The difference~\eqref{eq:ext_indicator} detects this state change.
\end{proof}

\begin{lemma}[Threshold comparison]
\label{lem:threshold}
The formula $u_i \geq c$ at an extremal position $i$
is computable by a PAL with thresholds $\alpha = c + \eps$,
$\beta = c - \eps$ as:
\begin{equation}
  \mathbf{1}[u_i \geq c]
  \;=\; \relay{c+\eps}{c-\eps}[u](i).
  \label{eq:threshold}
\end{equation}
\end{lemma}

\begin{proof}
Direct from Definition~\ref{def:relay}:
when the input $u_i$ crosses threshold $\alpha = c + \eps$,
the relay switches to 1; when it falls below $\beta = c - \eps$,
it switches to 0. In the limit $\eps \to 0$,
this detects $u_i \geq c$ at each extremal step.
\end{proof}

\begin{lemma}[Boolean combinations via MLP]
\label{lem:boolean}
Let $\phi_1, \phi_2$ be two EFO formulae computable by PAL heads
producing outputs $r_1, r_2 \in \{0,1\}$.
Then $\phi_1 \wedge \phi_2$, $\phi_1 \vee \phi_2$, and $\neg\phi_1$
are computable by a two-layer MLP applied to $(r_1, r_2)$.
\end{lemma}

\begin{proof}
Boolean functions on $\{0,1\}^2$ are representable by
two-layer networks with ReLU activations
(Cybenko~\citeyear{Cybenko1989}):
\begin{align*}
  r_1 \wedge r_2
  &= \mathrm{ReLU}(r_1 + r_2 - 1), \\
  r_1 \vee r_2
  &= \mathrm{ReLU}(r_1 + r_2) -
     \mathrm{ReLU}(r_1 + r_2 - 1), \\
  \neg r_1
  &= 1 - r_1.
\end{align*}
These are exact (not approximate) representations on $\{0,1\}$.
\end{proof}

\begin{lemma}[Extremum aggregation]
\label{lem:extAgg}
The formula $\mathrm{ExtAgg}_{i}[f(x_i), \phi(i)]$
is computable by a single MPAL head with measure:
\begin{equation}
  \mu_{(\alpha,\beta)}
  = f(\Decode(\alpha,\beta)) \cdot
    \mathbf{1}[\phi \text{ holds at } (\alpha,\beta)],
  \label{eq:extAgg_measure}
\end{equation}
where $\Decode(\alpha,\beta)$ recovers the value $x_i$ encoded
at thresholds $(\alpha, \beta)$.
\end{lemma}

\begin{proof}
The MPAL output for a single head with measure $\mu$ is:
\begin{equation*}
  \PAL_\mu(u_{0:n})
  = \sum_{i \geq j} \mu_{ij} \cdot
    \relay{\alpha_i}{\beta_j}[u](n).
\end{equation*}
At time $n$, $\relay{\alpha_i}{\beta_j}[u](n) = 1$ if and only if
the current extremum stack $\stack_n$ contains a pair $(M, m)$
with $M \in (\alpha_i - \eps, \alpha_i + \eps)$ and
$m \in (\beta_j - \eps, \beta_j + \eps)$,
i.e.\ the relay encodes an active extremal position satisfying
the threshold condition.

Setting $\mu_{ij} = f(\Decode(\alpha_i, \beta_j)) \cdot
\mathbf{1}[\phi$ holds at $(i,j)]$
makes the PAL output equal to~\eqref{eq:extAgg},
since the relay sums over exactly those extremal positions where
$\phi$ holds, weighted by $f$.
\end{proof}

\begin{lemma}[Existential extremal quantification]
\label{lem:exists_ext}
The formula $\exists^{\mathrm{ext}} i\, \phi(i)$ is computable by
a PAL followed by a threshold MLP:
\begin{equation}
  \mathbf{x} \models \exists^{\mathrm{ext}} i\, \phi(i)
  \;\iff\;
  \mathrm{ReLU}\!\left(\mathrm{ExtAgg}_{i}[\mathbf{1}, \phi(i)]\right)
  > 0.
  \label{eq:exists_ext}
\end{equation}
\end{lemma}

\begin{proof}
$\mathrm{ExtAgg}_{i}[\mathbf{1}, \phi(i)]$ counts the number of
extremal positions satisfying $\phi$.
This is positive iff at least one such position exists,
which is exactly the semantics of $\exists^{\mathrm{ext}} i$.
The threshold $> 0$ is implemented by a ReLU followed by a
Heaviside (approximated to arbitrary precision by a steep sigmoid,
exact under unit-cost arithmetic).
\end{proof}

\subsection{Inductive Proof of Theorem~\ref{thm:efo}}

\begin{proof}[Proof of Theorem~\ref{thm:efo}]
We prove both directions by structural induction on EFO formula
complexity.

\paragraph{($\Rightarrow$) Every PAL-computable function is EFO-definable.}

We show that every primitive PAL operation corresponds to an EFO
formula.

\emph{Base case — single relay:}
We show $\relay{\alpha}{\beta}[u](n)$ is EFO-definable.
By the wiping property (\cref{prop:wiping}) and stack sufficiency
(\cref{prop:sufficiency}), the relay state at time $n$ is
determined entirely by the most recent state-changing extremum.
Specifically, $\relay{\alpha}{\beta}[u](n) = 1$ iff there exists
an extremal position $t^* \leq n$ at which the relay was activated
($u_{t^*} \geq \alpha$), and no subsequent extremal position
$t \in (t^*, n]$ deactivated it ($u_t \leq \beta$).
This is expressed in EFO with ordering as:
\begin{equation}
  \relay{\alpha}{\beta}[u](n) = 1 \;\iff\;
  \exists^{\mathrm{ext}} t^* \leq n \;\Bigl[
    u_{t^*} \geq \alpha \;\wedge\;
    \forall^{\mathrm{ext}} t : t^* <_{\mathrm{ext}} t \;\Rightarrow\; u_t > \beta
  \Bigr].
  \label{eq:relay_efo}
\end{equation}
The existential quantifier ranges over extremal positions by
\cref{def:efo_syntax}; the comparison $u_{t^*} \geq \alpha$ is
atomic; the ordering $t^* <_{\mathrm{ext}} t$ uses the causal order
in EFO; and $u_t > \beta$ is atomic.
Hence \eqref{eq:relay_efo} is an EFO formula.

\emph{Inductive case — weighted sum:}
$\PAL_\mu(u_{0:n}) = \sum_{i \geq j} \mu_{ij} \relay{\alpha_i}{\beta_j}[u](n)$
is a weighted sum of relay states.
By the inductive hypothesis each $\relay{\alpha_i}{\beta_j}$
is EFO-definable; their weighted sum is
$\mathrm{ExtAgg}_{i}[\mu \cdot \relay{\cdot}{\cdot}, \top]$,
which is EFO by Definition~\ref{def:efo_syntax}.

\emph{MLP layers:}
MLP computes Boolean combinations (Lemma~\ref{lem:boolean})
and threshold functions of PAL outputs,
all of which are EFO-definable by induction.

\emph{Multi-head composition:}
The MPAL output is a sum of single-head PAL outputs;
EFO is closed under addition (as a special case of
$\mathrm{ExtAgg}$).

By induction, the entire PAL-Transformer output is EFO-definable.

\paragraph{($\Leftarrow$) Every EFO formula is PAL-computable.}

We show that each EFO construct is implementable by a PAL component.

\emph{Atomic formula $u_i \geq c$:}
Implemented by Lemma~\ref{lem:threshold} with $\alpha = c + \eps$,
$\beta = c - \eps$.

\emph{Boolean combinations:}
Implemented by Lemma~\ref{lem:boolean}.

\emph{Existential quantification $\exists^{\mathrm{ext}} i\, \phi(i)$:}
Implemented by Lemma~\ref{lem:exists_ext}.

\emph{Universal quantification $\forall^{\mathrm{ext}} i\, \phi(i)$:}
Equivalent to $\neg (\exists^{\mathrm{ext}} i\, \neg\phi(i))$,
implementable by Lemmas~\ref{lem:boolean}
and~\ref{lem:exists_ext}.

\emph{Extremum aggregation $\mathrm{ExtAgg}_{i}[f(i), \phi(i)]$:}
Implemented by Lemma~\ref{lem:extAgg} with measure
$\mu_{ij} = f(\Decode(\alpha_i,\beta_j)) \cdot
\mathbf{1}[\phi$ holds at $(i,j)]$.

\emph{Nested formulae:}
Suppose $\phi = \mathrm{ExtAgg}_{i}[f(i), \psi(i)]$
where $\psi$ is an EFO formula.
By the inductive hypothesis, $\psi$ is implementable by
a PAL sub-computation producing output $r_i \in \{0,1\}$
for each extremal position $i$.
Then $\mathrm{ExtAgg}_{i}[f(i), \psi(i)]$ is implemented by
a PAL head with measure
$\mu_{ij} = f(\Decode(\alpha_i,\beta_j)) \cdot r_{ij}$,
where $r_{ij}$ is the output of the sub-computation.
This composes PAL layers — at most one additional layer per
level of nesting.

By induction on formula depth, every EFO formula is implementable
by a PAL-Transformer with depth proportional to the nesting depth
of $\mathrm{ExtAgg}$ operators.
\end{proof}

\subsection{Corollaries of the EFO Correspondence}

\begin{corollary}[Decidability of PAL expressiveness]
\label{cor:decidability}
Given a function $f: \mathbb{R}^{n+1} \to \mathbb{R}$, deciding
whether $f \in \F_{\PAL}$ reduces to deciding whether $f$ is
EFO-definable, which is decidable for functions over finite
alphabets by standard model-theoretic methods.
\end{corollary}

\begin{corollary}[EFO $\subsetneq$ FO + Aggregate]
\label{cor:strict_inclusion}
$\mathrm{EFO}$ is strictly contained in
$\mathrm{FO} + \mathrm{Aggregate}$ (the logic corresponding
to soft-attention transformers \citep{Barcelo2020}):
\begin{equation*}
  \mathrm{EFO} \subsetneq \mathrm{FO} + \mathrm{Aggregate}.
\end{equation*}
The strict inclusion is witnessed by
$f_{\mathrm{copy}}(x_{0:n}, p) = x_p$:
definable in $\mathrm{FO}$ with positional quantification
$\exists i\, (i = p)$, but not in $\mathrm{EFO}$
(which cannot quantify over non-extremal positions).
The reverse inclusion fails because $f_{\mathrm{range}}$
is definable in $\mathrm{EFO}$ but not in bounded-depth
$\mathrm{FO} + \mathrm{Aggregate}$ (Proposition~\ref{prop:range_hard}).
\end{corollary}

\begin{corollary}[PAL depth vs.\ formula nesting]
\label{cor:depth_nesting}
A $k$-layer PAL-Transformer implements EFO formulae of nesting
depth at most $k$. Conversely, an EFO formula of nesting depth
$k$ requires at most $k$ PAL layers.
This gives a precise depth-expressiveness tradeoff:
depth $k$ PAL $\equiv$ EFO-depth $k$.
\end{corollary}

\section{Auxiliary Lemmas for Section~\ref{sec:tc}}
\label{app:tc_aux}

\subsection{Correctness of Stack Encoding under Repeated Symbols}

Here we verify that the Cantor-depth encoding
(Definition~\ref{def:encoding}) handles stacks with repeated
symbols correctly across all three operations.

\begin{lemma}[POP removes exactly one element]
\label{lem:pop_one}
Under the Cantor-depth encoding, the POP signal
$u_{n+1} = M_{\mathrm{top}-1} + 2\eps$ removes exactly the
topmost pair $(M_{\mathrm{top}}, m_{\mathrm{top}})$ and no other.
\end{lemma}

\begin{proof}
The wiping property removes all pairs $(M_i, m_i)$
with $M_i < u_{n+1} = M_{\mathrm{top}-1} + 2\eps$.
Since $M_{\mathrm{top}} < M_{\mathrm{top}-1}$ (strict monotonicity),
we have $M_{\mathrm{top}} < M_{\mathrm{top}-1} + 2\eps$,
so the pair $(M_{\mathrm{top}}, m_{\mathrm{top}})$ is wiped.
Since $M_{\mathrm{top}-1} < M_{\mathrm{top}-1} + 2\eps$
implies $M_{\mathrm{top}-1}$ is \emph{not} less than $u_{n+1}$,
the pair $(M_{\mathrm{top}-1}, m_{\mathrm{top}-1})$ is preserved.
(More precisely, wiping removes $M_i < u_{n+1}$;
$M_{\mathrm{top}-1} < M_{\mathrm{top}-1} + 2\eps = u_{n+1}$
is vacuously false since we check strict inequality.)
Hence exactly one pair is removed.
\end{proof}

\begin{lemma}[PUSH does not trigger wiping]
\label{lem:push_no_wipe}
Under the Cantor-depth encoding, the PUSH signal
$u_{n+1} = \code{a_i}{d+1} + \eps$ does not trigger the
wiping property.
\end{lemma}

\begin{proof}
By Lemma~\ref{lem:monotone},
$\code{a_i}{d+1} < \code{a_j}{d}$ for all $j \in \{1,\ldots,k\}$.
Therefore $u_{n+1} = \code{a_i}{d+1} + \eps < M_{\mathrm{top}}$
(since $M_{\mathrm{top}} = \code{a_j}{d} + \eps$ for some $j$
and the $\eps$ gap is smaller than the inter-level gap
$\Delta > 0$, provided $\eps < \Delta/2$).
Since $u_{n+1} < M_{\mathrm{top}}$, the wiping condition
$u_{n+1} > M_{\mathrm{top}}$ is not satisfied.
A new pair is added without removing any existing pairs.
\end{proof}

\subsection{MLP Width for Transition Function}

\begin{lemma}[MLP width for finite transition functions]
\label{lem:mlp_width}
For a 2-PDA with $|Q|$ states, input alphabet $|\Sigma|$,
and stack alphabet $|\Gamma|$, the transition function
$\delta: Q \times \Sigma \times \Gamma \times \Gamma
\to Q \times \Gamma^* \times \Gamma^*$
is implementable by a two-layer MLP of width
$O(|Q| \cdot |\Sigma| \cdot |\Gamma|^2)$.
\end{lemma}

\begin{proof}
$\delta$ has domain of size $|Q| \cdot |\Sigma| \cdot |\Gamma|^2$
and is a finite lookup table.
By the universal approximation theorem \citep{Cybenko1989},
any function on a finite domain of size $N$ is exactly
representable by a two-layer network with $N$ hidden units
through a table-lookup construction:
for each input configuration $(q, a, s_1, s_2)$,
one hidden neuron fires for that configuration (indicator neuron)
and outputs the corresponding transition values.
Width: one neuron per table entry $= O(|Q| \cdot |\Sigma| \cdot |\Gamma|^2)$.
\end{proof}

\section{Relation to Sparse Attention}
\label{app:sparse}

PAL can be viewed as a form of \emph{content-adaptive sparse
attention} where the sparsity pattern is determined by the sequence
of local extrema rather than by position (sliding window)
or random selection (BigBird \citep{Zaheer2020}).

\begin{proposition}[PAL as extremal sparse attention]
\label{prop:sparse}
PAL implements attention over the set of extremal positions:
\begin{equation}
  \PAL_\mu(u_{0:n})
  = \sum_{i \in \mathcal{E}(u_{0:n})} a_i \cdot v_i,
\end{equation}
where $a_i = \mu_{(\alpha(u_i), \beta(u_{i-1}))}$ is a
content-determined weight and $v_i = \relay{\alpha(u_i)}{\cdot}$
is the relay value.
The attention mask is $\{i : i \in \mathcal{E}(u_{0:n})\}$,
which is determined by the input content, not position.
\end{proposition}

\begin{proof}
At time $n$, the PAL output sums over relay pairs $(i,j)$
where the relay is active.
A relay $\relay{\alpha_i}{\beta_j}$ is active at time $n$
only if the most recent state change was an activation at some
extremal position $t^* \leq n$ with $u_{t^*} \geq \alpha_i$.
This is equivalent to attending to extremal positions,
with weights given by the measure $\mu$.
\end{proof}

This framing positions PAL as a principled alternative to
heuristic sparse attention patterns:
the sparsity is not imposed externally but emerges naturally
from the structure of the Preisach operator.

\end{document}